\documentclass[runningheads]{llncs}
\usepackage[T1]{fontenc}
\usepackage{graphicx}
\usepackage{amsmath,amssymb}
\usepackage{algorithm}
\usepackage{algorithmic}
\usepackage{booktabs}
\usepackage[figuresright]{rotating}
\usepackage{subcaption}
\usepackage{color}
\usepackage{times}
\usepackage{xspace}
\usepackage{hyperref}

\hypersetup{
    colorlinks=true,
    linkcolor=blue,
    filecolor=magenta,      
    urlcolor=cyan,
    pdftitle={Overleaf Example},
    pdfpagemode=FullScreen,
    }

\newcommand{\gs}{GSEMO\xspace}
\newcommand{\gsemo}{GSEMO\xspace}
\newcommand{\sg}{SW-GSEMO\xspace}
\newcommand{\nsga}{NSGA-\uppercase\expandafter{\romannumeral2} }

\DeclareMathOperator*{\argmax}{argmax}

\begin{document}
\title{Sliding Window Bi-Objective Evolutionary Algorithms for Optimizing Chance-Constrained Monotone Submodular Functions}
\titlerunning{\sg for Optimizing Chance-Constrained Submodular Functions}
\author{Xiankun Yan\inst{1}\orcidID{0000-0002-2309-8034} \and
Aneta Neumann\inst{1}\orcidID{0000-0002-0036-4782} \and
Frank Neumann\inst{1}\orcidID{0000-0002-2721-3618}}

\authorrunning{Xiankun Yan, Aneta Neumann and Frank Neumann}
\institute{Optimisation and Logistics,
   School of Computer and Mathematical Sciences,\\
   The University of Adelaide, Adelaide, Australia \\
\email{\{xiankun.yan,aneta.neumann,frank.neumann\}@adelaide.edu.au}
}
\maketitle              % typeset the header of the contribution

\begin{abstract}
  Variants of the \gsemo algorithm using multi-objective formulations have been successfully analyzed and applied to optimize chance-constrained submodular functions.
  However, due to the effect of the increasing population size of the \gsemo algorithm considered in these studies from the algorithms, the approach becomes ineffective if the number of trade-offs obtained grows quickly during the optimization run.
  In this paper, we apply the sliding-selection approach introduced in \cite{Neumann2023fast}
  to the optimization of chance-constrained monotone submodular functions. 
  We theoretically analyze the resulting \sg algorithm which successfully limits the population size as a key factor that impacts the runtime and show that this allows it to obtain better runtime guarantees than the best ones currently known for the \gsemo.
  In our experimental study, we compare the performance of the \sg to the \gsemo and \nsga on the maximum coverage problem under the chance constraint and show that the \sg outperforms the other two approaches in most cases. In order to get additional insights into the optimization behavior of \sg, we visualize the selection behavior of \sg during its optimization process
  and show it beats other algorithms to obtain the highest quality of solution in variable instances.

\end{abstract}

\section{Introduction}
\label{sec:intro}
Evolutionary algorithms (EAs) have been successfully applied to solve a wide range of complex combinatorial optimization problems.
The algorithms have received both theoretical and empirical studies, showing their ability to obtain good solutions within a reasonably expected runtime for problems with deterministic settings~\cite{corus2013generalized,friedrich2009analyses,friedrich2007approximating,kratsch2010fixed,neumann2018runtime,roostapour2018performance,wu2016impact}.
Additionally, it has been observed that some EAs also perform effectively on stochastic and dynamic optimization problems.
Researchers are actively investigating the advantages and limitations of EAs in solving such problems, particularly within the field of evolutionary computation theory.
A variety of studies~\cite{lehre2021more,lissovoi2017runtime,neumann2015runtime,neumann2020analysis,roostapour2022pareto,shi2018runtime,yan2023optimizing} have been conducted, presenting both the challenges and the advanced implications of these algorithms.

In a stochastic environment, completely eliminating the negative and uncertain impact of stochastic components is challenging.
Therefore, it is crucial to minimize these negative effects to prevent unpredictable disruptions in most complex systems.
Chance constraint ~\cite{charnes1959chance,TCHTTP,iwamura1996genetic,miller1965chance,E23OMOEA,U3OEADCHKP,MOEAwSWS,poojari2008genetic} is a useful technique for handling the effects of stochastic circumstances. 
It allows the deterministic bound in the constraint to be violated, but only with a very small probability during optimization.
However, directly evaluating chance constraints is complex and time-consuming.
A practical approach is to transform the stochastic constraint into its corresponding deterministic equivalent for a given confidence level,
rather than statistically calculating the probability of violation when optimizing chance-constrained problems with known distribution elements~\cite{shi2022runtime,yan2023optimizing}.

Submodular functions~\cite{nemhauser1978analysis} represent various problems where the incremental benefit of adding solution elements diminishes with the increasing solution size. 
The optimization of submodular functions is significantly challenging and has been extensively investigated with different types of constraints in previous work~\cite{khuller1999budgeted,leskovec2007cost,nemhauser1978analysis,Neumann2023fast,roostapour2022pareto,yan2024sampling,yaroslavtsev2020bring}.  
In the paper, we study a chance-constrained version of the submodular optimization problem subject to a knapsack constraint.
The problem seeks to find a subset of stochastic elements that maximizes the submodular function value,
while ensuring that the actual weight exceeds a given bound with a very small probability.
In the previous research,
Doerr et al~\cite{doerr2020optimization} analyzed the performance of the greedy algorithms on the problem.
They constructed the surrogate functions using the tail inequalities (Chernoff bound and one-sided side Chebyshev’s inequality) to handle the chance constraint.
Their findings demonstrated the greedy algorithms can achieve $(1-o(1))(1-1/e)$-approximation and $(1/2-o(1)(1-1/e)$-approximation for the problem with the uniform independent and identically distributed (IID) weights and the uniform weights with the same dispersion respectively. 
Subsequently, Neumann et al.,~\cite{neumann2020optimising} applied the \emph{Global Simple Evolutionary Multi-objective Algorithm }(\gs) to the same problem.
They employed a similar surrogate approach for estimating violation probability as done in~\cite{doerr2020optimization} and considered both the submodular function value and the probability as objectives in their bi-objective fitness function.
Their theoretical analysis demonstrated that the \gs can achieve similar approximation results as \cite{neumann2020optimising} within an expected runtime related to the max population size that it can generate.
Additionally, their experimental results indicated that the \gs could obtain better solutions than the greedy algorithm and other evolutionary algorithms, in solving the problem involving uniform IID weights.

However, the runtime analysis presented in~\cite{neumann2020optimising} indicates that the growing population size significantly impacts the efficiency of the \gs.
This is particularly evident when dealing with uniform weights of the same dispersion, where the GSEMO can quickly maintain an exponential number of trade-off solutions in the population.
To address this inefficiency, the sliding-selection approach has been introduced,
leading to the development of an improved version of \gs,
named the Sliding Window \gs (\sg).
Neumann and Witt~\cite{Neumann2023fast} firstly presented the \sg's enhanced performance in solving submodular problems with deterministic weights.
In brief, the sliding-selection method involves defining a weight window of size one,
which is determined based on the given bound and the ratio of the current time to the total time.
A solution is eligible for selection as a parent if its weight falls within the current window; otherwise, the algorithm uniformly selects one individual from the current population. 
%They demonstrated that the algorithm with the help of the sliding window finds $(1-1/e)$-approximation solution within the runtime $O(nklogn)$ when the submodular function under a uniform constraint of size $k$.

Within this paper, the \sg is slightly updated in terms of the window selection method for optimizing chance-constrained monotone submodular functions.
The individuals that are in the window are still potential parents for the mutation,
however, when there is no individual in the window area, the algorithm chooses a current solution with the largest function value as the parent for the next mutation unless the current time exceeds the time budget. 
Following the approach in~\cite{neumann2020optimising}, 
the fitness function in our algorithm contains two objectives: 
the evaluated weight and the function value of the solution.
We also employ the same surrogate methods based on tail inequalities for weight evaluation.
Our study includes settings that include both IID weights,
and uniform weights with the same dispersion.
Our analysis reveals that the incorporation of the sliding-selection approach reduces the impact of growing population size on the expected runtime.
As a result, the expected runtime of the \sg to achieve similar expected approximation results shows improvement.
Additionally, we empirically assess the performance of the \sg on the maximum coverage problem under various settings, bounds, and violation probabilities.
We compare its results with those from the original \gs and the Non-Dominated Sorting Genetic Algorithm~\uppercase\expandafter{\romannumeral2} (\nsga).
The initial solutions of all algorithms are empty sets and the \nsga uses population sizes of $20$ and $100$ respectively. 
Furthermore, we provide a visualization of the sliding window's operation during optimization. 
Our investigation also seeks why algorithms using surrogates based on the Chernoff bound perform better than those using the one-sided Chebyshev's inequality when the probability is smaller.

% \aneta{Section 3 ?} - solved
The paper is structured as follows. Section~\ref{sec:pre} introduces the chance-constrained monotone submodular problem and the investigated settings. 
Section~\ref{sec:moea} describes the adopted multi-objective evolutionary algorithms.
Our theoretical runtime analysis and proofs of the \sg are presented in Section~\ref{sec:psg}. 
Then we investigate the performance of the different algorithms and visualize the selection behavior in Section~\ref{sec:exp}. 
Finally, we end with some conclusions in Section~\ref{sec:con}.

\section{Preliminaries}
\label{sec:pre}
% \frank{$V = \{v_1, \ldots,v_n\}$},-sovled
Given a set $V = \{v_1,...,v_n\}$,
we consider the optimization of a monotone submodular function $f: 2^V\to \mathbb{R}_{+}$.
A function is called monotone iff for every $S,T\subseteq V$ with $S\subseteq T$, $f(S)\leq f(T)$ holds.
A function $f$ is called submodular iff for every $S, T\subseteq V$ with $S\subseteq T$, $v_i\in V$ and $v_i\notin T$,
$f(S\cup \{v\})-f(S) \geq f(T\cup \{v\})-f(T)$ holds. 
We consider the optimization of such a function $f$ 
subjected to the chance constraint where each element $v_i$ takes on a random weight $W(v_i)$. 
Here, the chance-constrained  optimization problem can be formulated as
\begin{align*}
    Maximize & \quad f(S) \\
    S.t. & \quad Pr[W(S)>B]\leq \alpha,
\end{align*}
% \frank{You use $B$ as the bound here, but for uniform IID (Section 4.1) you use k!}-sovled
where $W(S) = \sum_{v_i\in S} W(v_i)$ is the total weight of the subset $S$,
and $B$ is the given deterministic bound. 
The parameter $\alpha$ quantifies the probability of exceeding the bound $B$ that can be tolerated. 

Following previous work~\cite{doerr2020optimization,neumann2020optimising},
we consider two different settings,
which are (1) \emph{Uniform IID Weights}: the weight of each element $v_i\in V$ is sampled from the uniform distribution with the same expected value $E_W(v_i)=a$ and dispersion $\delta(v_i) = d$ (i.e., $W(v_i) \in [a-d,a+d$] and $0<d\leq a$);
and (2) \emph{Uniform Weights with the Same Dispersion}: the weight of element $v_i\in V$ is sampled from the uniform distribution with the different expected value $E_W(v_i) = a_i$ but the same dispersion $\delta(v_i)=d$ (i.e., $W(v_i) \in [a_i-d, a_i+d]$ and $0<d\leq a_i$).
In both settings, we assume that the expected weight of each element is a positive integer. 
For further discussion, the subset $S$ is encoded as a decision vector $x = x_1, x_2,...,x_n$ with length $n$, where $x_i=1$ means that the element $v_i \in V$ is selected into the subset $S$.
Besides, we denote $|x|_1$ the number of elements packed into the subset $S$. 
Since all the settings are based on the uniform distribution, we have the expected weight and variance of the solution $x$ as
$$E[W(x)] = \sum_{i=0}^n E_W(v_i)x_i,$$
and
$$Var[W(x)] = d^2 |x|_1/3.$$
% \frank{Better use $|x|_1$ for the number of elements/ones}-solved

To handle the chance constraint, we establish the surrogate function based on \textit{One Chebyshev's inequality} and \textit{Chernoff bound} as described in the previous work~\cite{doerr2020optimization,neumann2020optimising}. 
The weight calculated by different surrogate functions are respectively formulated as
$$W_{cheb}(x,\alpha) = E[W(x)] + \sqrt{\frac{(1-\alpha)Var[W(x)]}{\alpha}},$$ and 
$$W_{chern}(x,\alpha) = E[W(x)] + \sqrt{3d|x|_1\ln{(1/\alpha)}}.$$
It had been proved that
if the surrogate weight of a solution $x$ is less than the bound $B$, then $x$ is feasible~\cite{doerr2020optimization}. 
% \frank{Why $W_{chef}$?}-sovled

\section{Multi-Objective Evolutionary Algorithm}
\label{sec:moea}

\begin{algorithm}[tb]
\raggedright
\caption{GSEMO}
\label{alg:gsemo}
\textbf{Input}: Probability $\alpha$, bound $B$ \\
% \textbf{Parameter}: Optional list of parameters\\
\textbf{Output}: the best individual $x$
\begin{algorithmic}[1] %[1] enables line numbers
\STATE Set $x = 0^n$;
\STATE $P\gets \{x\}$;
\REPEAT
\STATE Choose $x\in P$ uniformly at random;
\STATE $y\gets$ flip each bit of $x$ with probability $\frac{1}{n}$;
\IF{$\nexists w \in P : w\succ y$}
\STATE $P\gets (P\setminus\{z\in P | y \succeq z\})\cup \{y\}$;
\ENDIF
\UNTIL{stop;}
\end{algorithmic}
\end{algorithm}

\begin{algorithm}[tb]
\raggedright
\caption{SW-GSEMO}
\label{alg:sw_gsemo}
\textbf{Input}: Total time $t_{max}$, probability $\alpha$, bound $B$ \\
% \textbf{Parameter}: Optional list of parameters\\
\textbf{Output}: the best individual $x$
\begin{algorithmic}[1] %[1] enables line numbers
\STATE Set $x = 0^n$;
\STATE $P\gets \{x\}$;
\STATE $t \gets 0$;
\REPEAT
\STATE $t = t+1$
\STATE Choose $x =$sliding-selection$(P,t,t_{max},\alpha,B)$;
\STATE $y\gets$ flip each bit of $x$ with probability $\frac{1}{n}$;
\IF{$\nexists w \in P : w\succ y$}
\STATE $P\gets (P\setminus\{z\in P | y \succeq z\})\cup \{y\}$
\ENDIF
\UNTIL{$t\geq t_{max}$;}
\end{algorithmic}
\end{algorithm}

\begin{algorithm}[tb]
\raggedright
\caption{sliding-selection}
\label{alg:sliding}
\textbf{Input}: Population P, current iteration $t$, total time $t_{max}$, probability $\alpha$, bound $B$\\
% \textbf{Parameter}: Optional list of parameters\\
\textbf{Output}: the selected individual $x$
\begin{algorithmic}[1] %[1] enables line numbers
\IF{$t\leq t_{max}$}
\STATE $\hat{c}\gets (t/t_{max})\cdot B$;
\STATE $\widehat{P} = \{x\in P \mid \lfloor\hat{c}\rfloor\leq g_2(x,\alpha)\leq \lceil\hat{c}\rceil\}$;
\IF{$\widehat{P} = \emptyset$}
\STATE $P' \gets \{x\in P | g_2(x,\alpha)\leq \lfloor \hat{c}\rfloor\}$
\STATE $x\gets \argmax_{x'\in P'} g_1(x')$; \label{alg_line:max_x}
\ELSE
\STATE Choose $x\in \widehat{P}$ uniformly at random;
\ENDIF
\ELSE
% \STATE $\widehat{P}\gets P$;
\STATE Choose $x\in P$ uniformly at random;
\ENDIF

\STATE \textbf{Return} $x$;

\end{algorithmic}
\end{algorithm}

Within the paper,
we primarily investigate multi-objective evolutionary algorithms on the given monotone submodular problem.
Each solution is considered to be a two-dimensional search point in the objective space.
The bi-objective fitness function of the solution $x$ is expressed as
\begin{equation}
    g_1(x) = \left\{
    \begin{array}{ccl}
        f(x) & & {g_2(x,\alpha)\leq B}  \\
        -1 & & {g_2(x,\alpha)>B}
    \end{array}\right.
\end{equation}

\begin{equation}
    g_2(x,\alpha) =
        {W_{sg}}(x,\alpha),
\end{equation}
where $f(x)$ denotes the submodular function value of $x$,
$W_{sg}(x,\alpha)$ denotes the surrogate weight of $x$.
Let $y \in \{0,1\}^n$ be another solution in the search space.
We say that $x$ (weakly) dominates $y (x\succeq y)$ iff $g_1(x) \geq g_1(y)$ and $g_2(x,\alpha) \leq g_2(y,\alpha)$.
Note that the infeasible solution is strongly dominated by the feasible one because of the objective function $g_1$.
Besides, the objective function $g_2$ guides the generated solutions approach to the feasible search space. 

In previous work~\cite{neumann2020optimising}, the \gs (see Algorithm~\ref{alg:gsemo}) is studied on the chance-constrained monotone submodular problem. 
Here the GSEMO starts with an initial solution represented by a $0^n$ bitstring, signifying an empty set.
During the optimization, it maintains a set of non-dominated solutions, continually updating this set as new solutions are generated. 
In each iteration, the GSEMO randomly and uniformly selects an individual $x$ from the population to serve as a parent for creating offspring $y$ via a standard bit-flip operator, which flips each bit of $x$ independently with a probability of $1/n$.
The offspring $y$ is accepted into the population if it is not strictly dominated by any existing solution.
Subsequently, any solution in the population that is weakly dominated by $y$ is removed.

The modified \sg algorithm is outlined in Algorithm~\ref{alg:sw_gsemo}. 
It also begins with a $0^n$ bitstring, maintains a population of non-dominating individuals,
and employs the standard mutation to generate offspring. 
However, it differs from the \gs in its parent selection method,
which defines a sliding window (see Algorithm~\ref{alg:sliding}) to select the potential parents instead of applying the random uniform selection in the whole population directly.
Within this method, given the total runtime $t_{max}$ and the current time $t$, a current bound $\hat{c}$ is defined as $\hat{c} = t\cdot B/t_{max}$.
A window is established as an interval between $[\lfloor\hat{c}\rfloor,\lceil\hat{c}\rceil]$.
Unlike the deterministic case where it selects solutions based on deterministic weight,
the \sg chooses some eligible individuals whose surrogate weights fall within this window. 
The parent is then randomly selected from these eligible solutions.
As the current time $t$ increases, the \sg would select solutions with larger surrogate weights that are still within the window.
When no individual lies within the current window, 
the \sg generates a sub-population $P'$ where the solution $x$ from the original population has the surrogate weight $g_2(x)$ that is lower than $\lfloor\hat{c}\rfloor$.
Then the solution $x'\in P'$ with the largest $g_1$-value is assigned for the mutation. 
Moreover, when $t\geq t_{max}$, the \sg reverts to selecting parents from the entire population $P$, similar to the \gs.

% \frank{You need to clearly say that the algorithms come from the ECAI paper and you would need to state clearly what is your new contribution.}-sovled 

\section{Performance of SW-GSEMO based on Surrogate}
\label{sec:psg}
To illustrate that the \sg works more efficiently in optimizing the chance-constrained monotone submodular problem than the \gs, the expected runtime of \sg with the surrogates to reach the same expected result for the problem is analyzed in this section. 
\subsection{Uniform IID Weights}
\label{sec:iid}
The previous work~\cite{neumann2020optimising} illustrates that the classical \gs reaches a $(1-o(1))(1-1/e)$-approximation for the chance-constrained problem with uniform IID weights in the expected time $O(nk(k+\log n))$,
where $k = min\{n+1,\lfloor (B/a)+1 \rfloor\}$. 
However, with the help of the sliding-selection, the size of the potential parents is bounded,
so the time to get a result with the same approximation is optimized. 

\begin{theorem}
\label{thm:1}
Consider \sg with $t_{max} = ekn\ln {(nk)}$ on a monotone submodular function $f$ under a chance constraint with uniform IID weights. 
% \frank{Statement "with probability 1-o(1) expected time" doesn't make sense. }-sovled
Then with probability $1-o(1)$,
the time that the algorithm finds a solution is no worse than $(1-o(1))(1-1/e)$-approximation is bounded by $O(nk\log {n})$ if $\lfloor B/a \rfloor = \omega(1)$.
\end{theorem}
\begin{proof}
    Let $k_{opt} = \lfloor B/a \rfloor$.
    As far as the previous work~\cite{neumann2020optimising} proved,
    the initial bitstring $0^n$ will stay in the population 
    since it is the best individual with respect to $g_2(0^n,\alpha) = 0$.
    Besides, note that $g_2$ is a strictly monotonically increasing function with the number of elements, 
    and solutions with the same number of elements have the same $g_2$-value because of the IID weights. 
    From~\cite{neumann2020optimising},
    it says that the element $x_j$ with the largest marginal gain $g_1(x\cup \{x_j\})-g_1(x)$ to the solution $x$ is picked up in the mutation with the probability $\Omega(1/en)$, 
    the solution $x^* =\{x_1,...,x_{k^*}\}$ that includes the largest $k^*$ elements is kept in population and satisfies the chance constraint, 
    and $x^*$ holds 
    $$f(x^*)\geq (1-(1-k_{opt})^{k^*})\cdot f(OPT).$$
    Note that $k^*<k_{opt}\leq k$ because of the chance constraint. 

    We denote $x^{*}_j$ a subset of $x^*$ with first $j$ elements (i.e., $1\leq j\leq k^*$ and $x^{*}_j = \{x_1,...,x_j\}\subseteq x^*$).
    First, we consider how the \sg includes the element $x_{j+1} \in x^{*}$ when the solution $x^{*}_j$ is not dominated by other solutions from the population. 
    Recall that the solution with an empty set is kept in the population already at the beginning.
    Besides, the surrogate weight of $x^{*}_j$ is increasing with respect to $j$.
    We assume that the solution $x^*_j$ is in the population at time $t_j := enW_{sg}(x^*_j)\ln{(nk)}$. 
    By definition of $\widehat{P}$,
    $x^*_j$ is available for the selection up to time 
    $$t_{j+1} - 1 = e(W_{sg}(x^*_j)+1)n\ln{(nk)}-1$$ 
    since  
    $$\left\lfloor\frac{e(W_{sg}(x^*_j)+1)n\ln{(nk)}-1}{t_{max}}\cdot B\right\rfloor = \lfloor W_{sg}(x^*_j)\rfloor.$$
    Furthermore, since $0< d\leq a$ and $\lceil\hat{c}\rceil - \lfloor\hat{c}\rfloor = 1$, the size of $\widehat{P}$ is at most 1 with the effect from the dispersion $d$. 
    Consequently,
    the probability of choosing the subset $x_j^*$ and mutating the element $x_{j+1}$ is at least $1/en$ between the time $t_j$ and $t_{j+1}-1$,
    i.e., for a period of $en\ln{(nk)}$ steps, the probability of those events not processing is at most
    $$(1-1/(en))^{en\ln{(nk)}} \leq 1/nk.$$
% \frank{$\ln(nk)$!!-solved}
    Then, we study the case where the solution $x^{*}_j$ is dominated by an individual solution $y$ with a larger $g_1$-value from the population.
    That also means $x^{*}_j$ is not in the current window at time $t_j$.
    We denote $y'$ as the solution if it exists in the current window. 
    Regarding the domination scheme of the algorithm, it implies that $g_1(x^{*}_j)\leq g_1(y) < g_1(y')$ and $g_2(y,\alpha) < g_2(x^{*}_j,\alpha) \leq g_2(y',\alpha)$. 
    Considering the Line~\ref{alg_line:max_x} of sliding-selection in Algorithm~\ref{alg:sliding}, 
    the \sg uses the solution $y'$ (or $y$ when no individual is in the current window) for the next mutation. 
    With the probability of at least $1/en$, the element $x_{j+1}$ can be added to the selected solution within the same time period. 
    Consequently, the failure probability is also at most $1/kn$.
    Besides, the feasible offspring still satisfies $ f(y' \cup \{x_{j+1}\} ) \geq f(x^*_{j+1}) \geq (1-(1-k_{opt})^{j+1})\cdot f(OPT)$.
    
    By a union bound over the at most $k$ required successes, the probability of missing including $x_{j+1}$ to the solution in the period by time $t_j$ is $O(1/n)$. 
    Then following~\cite{neumann2020optimising}, applying the surrogate for bounding the value of $k^*$ can get $(1-o(1))(1-1/e)$-approximation. \qed
\end{proof}

\subsection{Uniform Weights with the Same Dispersion}
Following the definition from \cite{neumann2020optimising},
we use the objective function $\hat{g_2} = E_W(x)$ instead of $g_2$ and the already mentioned objective function $g_1$ in the two-dimensional fitness function.
Thus, the fitness of a solution $x$ is evaluated by $\hat{g}(x) = (g_1(x),\hat{g}_2(x))$.
Consequently, it has another solution $y$ such that  $y$ is (weakly) dominated by $x$
iff $g_1(x)\geq g_1(y)$ and $\hat{g_2}(x)\leq \hat{g_2}(y)$. 
Note that $\hat{g}_2(x)$ will be also used in Algorithm~\ref{alg:sliding}. 

Additionally,
let $a_{max} = \max_{v_i\in V} a_i$, $a_{min} = \min_{v_i\in V}a_i$, and $0 < \delta \leq a_{min}$. 
Then we show that the \sg can get a result with $(1/2-o(1))(1-1/e)$-approximation in an efficient runtime if the solution has at least one element (i.e., $B/a_{max} = \omega(1)$) in the following theorem. 

\begin{theorem}
\label{thm:2}
    Let $P_{max}$ be the maximal size of the population. 
    Consider \sg with $t_{max} = 2en(B/a_{min})\ln(nB/a_{min})$ on a monotone submodular function under a chance constraint with uniform weights having the same dispersion.
    Then with probability $1-o(1)$, a solution that is no worse than $(1-1/2)(1-o(1))$-approximation is obtained within 
    % \frank{you need to define what you mean by quality}
    % \frank{delete expected}-solved
    time $O(n((B/a_{min})\log n+P_{max}))$. 
\end{theorem}

\begin{proof}
    Following the proof of theorem 2 in~\cite{neumann2020optimising}, we also adopt $\hat{g}_2^*$, the maximal $\hat{g}_2$-value for which $\hat{g}_2\leq \hat{g}_2^*$, to track the progress of the \sg. Recall that the bitstring $0^n$ exists in the population at the beginning and $\hat{g}_2^*$ is non-decreasing. 
    Considering a chosen solution $x$ for mutation, 
    the algorithm flips a 0-bit of $x$ corresponding to the largest ratio between the additional gain in $g_1$ and $\hat{g}_2$. After mutation, the generated solution $y$ holds 
    $$g_1(y)\geq \left[1-\left(1-\frac{\hat{g}_2^*+a_{min}}{B(k+1)}\right)^{k+1}\right]\cdot f(OPT),$$
    where $k = |y|_1$ and $OPT$ is the deterministic optimal solution for the problem with the deterministic uniform weight setting. 
    Note that $\hat{g}_2^*$ increase by at least $a_{min}$.

    First, we consider such solution $x$ is not dominated by any solution in the current population. We prove that the solution $x$ is chosen and such mutation in $x$ happens with a high probability in the \sg.
    Note that the $\hat{g}_2(x)$ is growing with respect to its size of elements $k$.
    By the definition of the subset $\widehat{P}$,
    such mutation happens in the time between $2en\hat{g}_2(x)\ln{(nB/a_{min})}$ to $2en(\hat{g}_2(x)+1)\ln{(nB/a_{min})}-1$,
    since
    $$\left\lfloor\frac{2en(\hat{g}_2(x)+1)\ln{(nB/a_{min})}-1}{t_{max}}\cdot B\right\rfloor = \hat{g}_2(x).$$

    % \frank{$\ln{(nB/a_{min})}$!!}-solved
    Thus, the available period is $2en\ln{(nB/a_{min})}$.
    Since $\lceil\hat{c}\rceil - \lfloor\hat{c}\rfloor = 1$,
    the size of $\widehat{P}$ consequently is bound by $2$ as $\hat{g}_2(x) = E_w(x)$ (recall the setting where all the expected weighs are positive integers).
    Then we have such one bit of flipping that occurs with a probability of at least $1/2en$. 
    Furthermore, the probability of the mutation that does not happen in the period is bounded by
    $$(1-1/(2en))^{2en\ln{(nB/a_{min})}} \leq \frac{1}{n(B/a_{min)}}.$$

    Now we investigate the case when there is a solution $y'$ dominates the solution $x$. That means $x$ does not exist in the current window at the time $2en\hat{g}_2(x)\ln{(nB/a_{min})}$. 
    Recall that there are at most two individuals in the current window. We denote them by $y'_1$ and $y'_2$, 
    which are satisfies $g_1(x)\leq g_1(y')< g_1(y'_1)<g_1(y'_2)$ and  $\hat{g}_2(y')\leq \hat{g}_2(x)= \hat{g}_2(y'_1)<\hat{g}_2(y'_2)$ (or $\hat{g}_2(y')\leq \hat{g}_2(y'_1)< \hat{g}_2(x)= \hat{g}_2(y'_2)$).
    Regarding the sliding selection, 
    Both solutions $y'_1$ and $y'_2$ (or $y'$ when no individual is in the current window) are good to be selected for the mutation since their function values are larger than the value of solution $x$. 
    Therefore, such a 1-bit flipping to get a qualified solution $y$ is under the probability $1/en$.
    Within the same period, the probability of failure is bounded by $o(a_{min}/{nB})$.
    
    By a union bound over at most $B/a_{min}$ required successes, the probability of failing to achieve at least one success is at most $1/n$. Therefore, the solution $y$ can be obtained in the \sg with the probability of at least $1-1/n=1-o(1)$.

    Let $x^*$ be the feasible solution with $|x^*|_1=k^*$ in the population, which has the largest surrogate weights.
    Here we consider the single element $v^*$ having the largest $g_1$-value but not included in $x^*$. 
    Following the work~\cite{neumann2020optimising}, the algorithm can obtain a solution $x'$ that only contains $v^*$ from the initial solution $0^n$ by flipping only one zero bit in the expected time $O(P_{max}n)$,
    where $P_{max}$ is the maximal size of population for the algorithm. 
    Then after applying the surrogate for bounding the value of $k^*$ as desired in \cite{neumann2020optimising}, the quality of $x^*$ or $x'$ is $(1/2-o(1))(1-1/e)$. 
    Finally, the expected time of \sg to get the expected result is $O(n((B/a_{min})\log n+P_{max}))$. \qed
\end{proof}

Backing into the special case of uniform IID weights, 
we note that $a = a_{max} = a_{min}$, 
the size of the population is at most $k = \min \{n+1,\lfloor B/a\rfloor+1\}$,
and the solution $x^*$ already contains the element with the largest $g_1$-value. Besides, as the effect of the uncertainty, the window only includes at most 1 solution.
Therefore, $x^*$ gives a $(1-o(1))(1-1/e)$- approximation and the expected time to get $x^*$ is bounded by $O(nk\log n)$,
which matches the result proved in Theorem~\ref{thm:1}.

\section{Experiments}
\label{sec:exp}

The experimental investigations of the \sg based on the different surrogate functions are proposed here. 
The results are compared with those generated from the \gs and the \nsga in various instances.
\subsection{Experimental setup}
The maximum coverage problem (MCP) based on the graph~\cite{feige1998threshold,khuller1999budgeted} is studied in the experiments, which is one of the famous submodular combinatorial optimization problems.
For the MCP, given an undirected graph $G = \{V,E\}$ with $n = |V|$ nodes, we denote $N(V')$ the number of all nodes of $V'\subseteq V$ and their neighbors in the graph $G$.
The goal of the problem is to find a subset of nodes $V'$ so that the nodes in the subset can cover more to their neighbors and themselves under the chance constraint with a deterministic bound $B$. Also, each node $v\in V$ has a stochastic weight $W(v)$. Therefore, the chance-constrained version of the problem is formulated as
\begin{equation}
    \argmax_{V'\subseteq V} N(V')~s.t~Pr[W(V')>B]\leq \alpha.
\end{equation}

Here some larger sparse graphs from the network data repository~\cite{rossi2015network} are considered to construct the instance of MCP.
The previous works~\cite{neumann2020optimising} studied the performance of \gs on the problem based on some small and dense graphs and showed it can easily cover most nodes even given a small bound. 
Thus, utilizing the larger sparse graphs can help us easily compare the results between the different algorithms.
Those graphs \emph{ca-CSphd}, \emph{ca-GrQc}, and \emph{ca-ConMat} are applied, which respectively contain $1,882$, $4,158$, and $21,363$ nodes. 

For the experiments under uniform IID weights,
each node $v$ is assigned a unit expected weight (i.e., $a=1$) and the dispersion $d = 0.5.$
Additionally, for the experiments considering uniform weights with the same dispersion, 
the expected weight of each node is based on its degree, which is expressed as $a_i = D(v_i)+1$ where $D(v_i)$ is the degree of $v_i$ in graph $G$.
For the dispersion, we set $d = 1$ to ensure $d\leq a_i$. 
Besides, we employ 
$B \in \{ \sqrt{n},\lfloor n/20\rfloor, \lfloor n/10 \rfloor \}$
to ensure that the bounds are proportional to the number of nodes in the graphs under study.

For all experiments, the problem is tested with $\alpha \in [0.001,0.1]$. 
In terms of the \gs and the \sg, the total iterations (or total time) are considered different as 
$t_{max} \in \{500000, 1000000, 1500000\}.$
Regarding the \nsga, the initial solutions are set to the $0^n$ strings. Its population sizes are set as $20$ and $100$ with the numbers of generated children $10$ and $50$ respectively. 
Thus, to keep the same fitness evaluation counts as other algorithms,
the max iterations for the \nsga are $t_{max}/10$ and $t_{max}/50$ respectively.
Moreover, we adopt the Kruskal-Wallis test with 95\% confidence in order to assess the statistical validity of our results. 
The Bonferroni post-hoc statistical procedure is employed for multiple comparisons of a control algorithm~\cite{corder2011nonparametric}.
For the given instance, $X^{(+)}$ is equivalent to the statement that the algorithm in the column is statistically better than the algorithm $X$ for the tested instance. 
Conversely, $X^{(-)}$ is equivalent to the statement that $X$ outperformed the algorithm, 
while $X^{(=)}$ demonstrates that the algorithm given in the column and $X$ have a comparable performance.

\begin{table}[tb]
\centering
\caption{Results for maximum coverage problem with IID weight}
\label{subTable:r_iid}
\resizebox{\textwidth}{0.31\textwidth}{
\begin{tabular}{@{}lllllllllllllllll@{}}
\toprule
            &           &      &           &          & \multicolumn{3}{l}{\gs with $W_{cheb}$   (1)} & \multicolumn{3}{l}{\sg with   $W_{cheb}$ (2)} & \multicolumn{3}{l}{$\nsga_{20}$ with   $W_{cheb}$(3)} & \multicolumn{3}{l}{$\nsga_{100}$ with   $W_{cheb}$(4)} \\ \midrule
Graph       & Surrogate & $B$  & $t_{max}$ & $\alpha$ & Mean          & std         & stat              & Mean                 & std      & stat             & Mean           & std           & stat                & Mean            & std           & stat                \\
ca-CondaMat & Chebyshev & 146  & 1500000   & 0.1      & 5588          & 47.265      & 2(-),3(+),4(-)    & \textbf{6790.966}    & 12.335   & 1(+),3(+),4(+)   & 5187.966       & 95.708        & 1(-),2(-),4(-)      & 6330.3          & 36.893        & 1(+),2(-),3(+)      \\
            &           &      &           & 0.001    & 4153.733      & 43.338      & 2(-),3(+),4(-)    & \textbf{4748.166}    & 7.585    & 1(+),3(+),4(+)   & 3802.733       & 75.95         & 1(-),2(-),4(-)      & 4531.83         & 33.061        & 1(+),2(-),3(+)      \\
            &           &      & 1000000   & 0.1      & 5500.73       & 47.67       & 2(-),3(+),4(-)    & \textbf{6771.366}    & 16.3     & 1(+),3(+),4(+)   & 4945.9         & 113.293       & 1(-),2(-),4(-)      & 6281            & 46.322        & 1(+),2(-),3(+)      \\
            &           &      &           & 0.001    & 3957.5        & 40.782      & 2(-),3(+),4(-)    & \textbf{4736.133}    & 10.375   & 1(+),3(+),4(+)   & 3709.06        & 96.656        & 1(-),2(-),4(-)      & 4496.7          & 28.765        & 1(+),2(-),3(+)      \\
            &           &      & 500000    & 0.1      & 4818.53       & 47.71       & 2(-),3(+),4(-)    & \textbf{6708.366}    & 27.316   & 1(+),3(+),4(+)   & 4685.333       & 114.8373      & 1(-),2(-),4(-)      & 6115.233        & 54.393        & 1(+),2(-),3(+)      \\
            &           &      &           & 0.001    & 3648.3        & 65.33       & 2(-),3(+),4(-)    & \textbf{4581.633}    & 36.88    & 1(+),3(+),4(+)   & 3469.766       & 73.953        & 1(-),2(-),4(-)      & 4395.566        & 43.076        & 1(+),2(-),3(+)      \\
            &           & 1068 & 1500000   & 0.1      & 11787.533     & 57.133      & 2(-),3(+),4(-)    & \textbf{16650.933}   & 14.163   & 1(+),3(+),4(+)   & 12284.866      & 142.248       & 1(-),2(-),4(-)      & 13394.133       & 75.448        & 1(+),2(-),3(+)      \\
            &           &      &           & 0.001    & 10893.9       & 61.683      & 2(-),3(+),4(-)    & \textbf{15217.3}     & 14.45    & 1(+),3(+),4(+)   & 10950.3        & 122.3         & 1(-),2(-),4(-)      & 12241.966       & 72.716        & 1(+),2(-),3(+)      \\
            &           &      & 1000000   & 0.1      & 11194.23      & 72.488      & 2(-),3(-),4(-)    & \textbf{16573.6}     & 19.608   & 1(+),3(+),4(+)   & 11623.93       & 119.467       & 1(+),2(-),4(-)      & 13164.2         & 60.07         & 1(+),2(-),3(+)      \\
            &           &      &           & 0.001    & 10364.53      & 58.95       & 2(-),3(+),4(-)    & \textbf{15145.666}   & 14.485   & 1(+),3(+),4(+)   & 9987.13        & 123.936       & 1(-),2(-),4(-)      & 12040.166       & 99.78         & 1(+),2(-),3(+)      \\
            &           &      & 500000    & 0.1      & 9474.733      & 111.851     & 2(-),3(-),4(-)    & \textbf{16368.6}     & 27.005   & 1(+),3(+),4(+)   & 10729.133      & 180.573       & 1(+),2(-),4(-)      & 12708.533       & 82.447        & 1(+),2(-),3(+)      \\
            &           &      &           & 0.001    & 9344.733      & 84.931      & 2(-),3(-),4(-)    & \textbf{14947.6}     & 22.553   & 1(+),3(+),4(+)   & 9502.2         & 142.003       & 1(+),2(-),4(-)      & 11581.066       & 84.55         & 1(+),2(-),3(+)      \\
            &           & 2136 & 1500000   & 0.1      & 12749.533     & 93.378      & 2(-),3(-),4(-)    & \textbf{20078.833}   & 11.066   & 1(+),3(+),4(+)   & 16361.766      & 67.76         & 1(+),2(-),4(-)      & 16243.166       & 68.601        & 1(+),2(-),3(+)      \\
            &           &      &           & 0.001    & 12730.3       & 91.025      & 2(-),3(-),4(-)    & \textbf{19327.433}   & 11.221   & 1(+),3(+),4(+)   & 15224.733      & 66.313        & 1(+),2(-),4(-)      & 15402.133       & 96.35         & 1(+),2(-),3(+)      \\
            &           &      & 1000000   & 0.1      & 11520.966     & 129.843     & 2(-),3(-),4(-)    & \textbf{20016.2}     & 16.172   & 1(+),3(+),4(+)   & 15696.5        & 116.216       & 1(+),2(-),4(-)      & 16020.633       & 88.739        & 1(+),2(-),3(+)      \\
            &           &      &           & 0.001    & 11500.633     & 114.875     & 2(-),3(-),4(-)    & \textbf{19245.233}   & 15.532   & 1(+),3(+),4(+)   & 14544.63       & 98.014        & 1(+),2(-),4(-)      & 15155.4         & 74.248        & 1(+),2(-),3(+)      \\
            &           &      & 500000    & 0.1      & 9488.333      & 90.631      & 2(-),3(-),4(-)    & \textbf{19822.566}   & 20.619   & 1(+),3(+),4(+)   & 14592.466      & 104.587       & 1(+),2(-),4(-)      & 14734.2         & 230.948       & 1(+),2(-),3(+)      \\
            &           &      &           & 0.001    & 9483.6        & 88.03       & 2(-),3(-),4(-)    & \textbf{19030.533}   & 21.451   & 1(+),3(+),4(+)   & 13456.133      & 89.165        & 1(+),2(-),4(-)      & 14399.233       & 152.192       & 1(+),2(-),3(+)      \\ \midrule
ca-CondaMat & Chernoff & 146  & 1500000   & 0.1      & 5321.1333     & 40.291      & 2(-),3(+),4(-)    & \textbf{6424.2}      & 10.403   & 1(+),3(+),4(+)   & 4931.833       & 82.145        & 1(-),2(-),4(-)      & 6018.066        & 41.946        & 1(+),2(-),3(+)      \\
            &           &      &           & 0.001    & 5027.166      & 49.31       & 2(-),3(+),4(-)    & \textbf{5994.833}    & 10.96    & 1(+),3(+),4(+)   & 4622.866       & 96.07         & 1(-),2(-),4(-)      & 5652.833        & 41.121        & 1(+),2(-),3(+)      \\
            &           &      & 1000000   & 0.1      & 5059.6        & 39.678      & 2(-),3(+),4(-)    & \textbf{6397.966}    & 14.943   & 1(+),3(+),4(+)   & 4755.133       & 73.411        & 1(-),2(-),4(-)      & 5954.966        & 50.251        & 1(+),2(-),3(+)      \\
            &           &      &           & 0.001    & 4784.766      & 54.93       & 2(-),3(+),4(-)    & \textbf{5979.9}      & 16.912   & 1(+),3(+),4(+)   & 4441.5         & 126.327       & 1(-),2(-),4(-)      & 5582.666        & 38.694        & 1(+),2(-),3(+)      \\
            &           &      & 500000    & 0.1      & 4625.4        & 60.563      & 2(-),3(+),4(-)    & \textbf{6328.2}      & 31.971   & 1(+),3(+),4(+)   & 4443.2         & 84.517        & 1(-),2(-),4(-)      & 5787.266        & 63.911        & 1(+),2(-),3(+)      \\
            &           &      &           & 0.001    & 4344.33       & 59.972      & 2(-),3(+),4(-)    & \textbf{5898.133}    & 22.47    & 1(+),3(+),4(+)   & 4170           & 103.826       & 1(-),2(-),4(-)      & 5437.7          & 43.076        & 1(+),2(-),3(+)      \\
            &           & 1068 & 1500000   & 0.1      & 11632.833     & 52.573      & 2(-),3(-),4(-)    & \textbf{16650.933}   & 14.163   & 1(+),3(+),4(+)   & 12054.233      & 126.656       & 1(+),2(-),4(-)      & 13206.26        & 65.605        & 1(+),2(-),3(+)      \\
            &           &      &           & 0.001    & 11464.966     & 61.738      & 2(-),3(-),4(-)    & \textbf{15217.3}     & 14.45    & 1(+),3(+),4(+)   & 11783.9        & 95.969        & 1(+),2(-),4(-)      & 12974.86        & 85.587        & 1(+),2(-),3(+)      \\
            &           &      & 1000000   & 0.1      & 11059.2       & 73.482      & 2(-),3(-),4(-)    & \textbf{16343.833}   & 16.806   & 1(+),3(+),4(+)   & 11441.2        & 145.704       & 1(+),2(-),4(-)      & 12961.266       & 85.324        & 1(+),2(-),3(+)      \\
            &           &      &           & 0.001    & 10914.833     & 76.11       & 2(-),3(-),4(-)    & \textbf{16052.866}   & 19.687   & 1(+),3(+),4(+)   & 11208.1        & 91.525        & 1(+),2(-),4(-)      & 12779.033       & 59.812        & 1(+),2(-),3(+)      \\
            &           &      & 500000    & 0.1      & 9482.966      & 92.698      & 2(-),3(-),4(-)    & \textbf{16129.133}   & 27.284   & 1(+),3(+),4(+)   & 10487.566      & 128.411       & 1(+),2(-),4(-)      & 12489.3         & 88.034        & 1(+),2(-),3(+)      \\
            &           &      &           & 0.001    & 9466.433      & 113.403     & 2(-),3(-),4(-)    & \textbf{15840.433}   & 26.94    & 1(+),3(+),4(+)   & 10254.133      & 130.19        & 1(+),2(-),4(-)      & 12309.4         & 81.884        & 1(+),2(-),3(+)      \\
            &           & 2136 & 1500000   & 0.1      & 12719.533     & 106.011     & 2(-),3(-),4(-)    & \textbf{19955.3}     & 10.312   & 1(+),3(+),4(+)   & 16187.466      & 93.583        & 1(+),2(-),4(-)      & 16130.1         & 80.711        & 1(+),2(-),3(+)      \\
            &           &      &           & 0.001    & 12701.7       & 90.339      & 2(-),3(-),4(-)    & \textbf{19813.8}     & 14.041   & 1(+),3(+),4(+)   & 16006.8        & 109.226       & 1(+),2(-),4(-)      & 15964.933       & 75.003        & 1(+),2(-),3(+)      \\
            &           &      & 1000000   & 0.1      & 11481.066     & 98.619      & 2(-),3(-),4(-)    & \textbf{19891.466}   & 13.197   & 1(+),3(+),4(+)   & 15519.3        & 102.403       & 1(+),2(-),4(-)      & 15812.433       & 93.623        & 1(+),2(-),3(+)      \\
            &           &      &           & 0.001    & 11458.1       & 116.396     & 2(-),3(-),4(-)    & \textbf{19750.366}   & 12.084   & 1(+),3(+),4(+)   & 15323.266      & 103.988       & 1(+),2(-),4(-)      & 15684.533       & 83.137        & 1(+),2(-),3(+)      \\
            &           &      & 500000    & 0.1      & 9451.033      & 105.529     & 2(-),3(-),4(-)    & \textbf{19697.033}   & 17.995   & 1(+),3(+),4(+)   & 14435.233      & 101.01        & 1(+),2(-),4(-)      & 14650.033       & 289.611       & 1(+),2(-),3(+)      \\
            &           &      &           & 0.001    & 9460.666      & 116.783     & 2(-),3(-),4(-)    & \textbf{19540.266}   & 18.77    & 1(+),3(+),4(+)   & 14218.8        & 129.106       & 1(+),2(-),4(-)      & 14586.566       & 214.181       & 1(+),2(-),3(+)      \\ \bottomrule
\end{tabular}
}
\end{table}

\subsection{Experimental results}

In the experiments, we first investigate the performance of \sg on the problem with different settings. 
Then, we visualize and demonstrate the behavior of the sliding-selection approach working in the optimization. 

\subsubsection{Results comparison}

Table~\ref{subTable:r_iid} and~\ref{subTable:r_uwd} displays the final function values achieved by the algorithms in the problem with different settings, which vary based on the surrogate functions used, and are tested across different iterations and values of $\alpha$.
Additional results are provided in Tables~\ref{table:iid_cheb}, \ref{table:iid_chef}, \ref{table:uwd_cheb} and \ref{table:uwd_chf} in Appendix\footnote{Appendix is in \href{https://doi.org/10.5281/zenodo.10984210}{here}.}.
Overall, for smaller graphs with a lower bound $B$, the \sg, \gs, and \nsga with different population sizes perform comparably. 
However, the \nsga can get better results when the expected weights are uniform, which is reflected in Table~\ref{subTable:r_uwd}.
For larger graph instances with larger bounds, the \sg gradually surpasses the other algorithms in performance. The mean and standard deviation of the \sg's results are also comparable to those of other algorithms across instances with varying $t_{max}$.
It is observed that as $t_{max}$ increases, the general performance of all algorithms improves. Moreover, while the function value differences for various $\alpha$ are not significant in larger graphs with higher bounds, they become substantial in smaller graphs.
Interestingly, the performance of the \nsga with population sizes of 20 and 100 is superior to that of the \gs in larger, sparser graph cases. 
This observation contrasts with findings from previous work~\cite{neumann2020optimising}, which uses dense graphs.
In terms of evaluating the problem with chance constraints by different surrogate methods, the results from Table~\ref{subTable:r_iid} and \ref{subTable:r_uwd} suggest that algorithms employing the one-sided Chebyshev's inequality are better than those using the Chernoff bound when $\alpha$ is large. On the other hand, the performance of algorithms based on the Chernoff bound becomes worse with smaller $\alpha$ values.

\begin{table}[tb]
\centering
\caption{Results for maximum coverage problem with uniform weights with same dispersion}
\label{subTable:r_uwd}
\resizebox{\textwidth}{0.31\textwidth}{
\begin{tabular}{@{}lllllllllllllllll@{}}
\toprule
            &           &      &           &          & \multicolumn{3}{l}{\gs (9)}         & \multicolumn{3}{l}{\sg  (10)}           & \multicolumn{3}{l}{$\nsga_{20}$ (11)} & \multicolumn{3}{l}{$\nsga_{100}$ (12)}     \\ \midrule
Graph       & Surrogate & $B$  & $t_{max}$ & $\alpha$ & Mean     & std    & stat              & Mean              & std   & stat             & Mean      & std     & stat             & Mean             & std   & stat             \\
ca-CondaMat & Chebyshev & 146  & 1500000   & 0.1      & 141.266  & 0.679  & 10(=),11(=),12(=) & \textbf{141.8}    & 0.979 & 9(=),11(=),12(=) & 141.266   & 0.813   & 9(=),10(=),12(=) & \textbf{141.8}   & 0.979 & 9(=),10(=),11(=) \\
            &           &      &           & 0.001    & 122.9    & 3.014  & 10(-),11(-),12(-) & 125.933           & 0.249 & 9(+),11(=),12(=) & 125.7     & 0.458   & 9(+),10(=),12(=) & \textbf{126}     & 0     & 9(+),10(=),11(=) \\
            &           &      & 1000000   & 0.1      & 141.2    & 0.6    & 10(=),11(=),12(=) & \textbf{141.733}  & 0.963 & 9(=),11(=),12(=) & 141       & 0.632   & 9(=),10(=),12(=) & 141.533          & 0.884 & 9(=),10(=),11(=) \\
            &           &      &           & 0.001    & 122.23   & 3.051  & 10(-),11(-),12(-) & 125.9             & 0.3   & 9(+),11(=),12(=) & 125.366   & 1.048   & 9(+),10(=),12(=) & \textbf{126}     & 0     & 9(+),10(=),11(=) \\
            &           &      & 500000    & 0.1      & 141.133  & 0.498  & 10(=),11(=),12(=) & \textbf{141.8}    & 0.979 & 9(=),11(=),12(=) & 140.833   & 0.734   & 9(=),10(=),12(=) & 141.333          & 0.745 & 9(=),10(=),11(=) \\
            &           &      &           & 0.001    & 120.733  & 3.172  & 10(-),11(-),12(-) & 125.6             & 0.663 & 9(+),11(=),12(=) & 124.066   & 2.657   & 9(+),10(=),12(=) & \textbf{125.7}   & 0.458 & 9(+),10(=),11(=) \\
            &           & 1068 & 1500000   & 0.1      & 1037.266 & 1.093  & 10(-),11(+),12(=) & \textbf{1044.833} & 0.933 & 9(+),11(+),12(+) & 1015.333  & 4.706   & 9(-),10(-),12(-) & 1037.833         & 2.646 & 9(=),10(-),12(+) \\
            &           &      &           & 0.001    & 978.4    & 2.751  & 10(-),11(+),12(-) & 991.766           & 2.347 & 9(+),11(+),12(=) & 959.566   & 10.892  & 9(-),10(-),12(-) & \textbf{992.2}   & 3.664 & 9(+),10(=),11(+) \\
            &           &      & 1000000   & 0.1      & 1034.933 & 1.152  & 10(-),11(+),12(=) & \textbf{1044.133} & 1.231 & 9(+),11(+),12(+) & 1012.333  & 5.204   & 9(-),10(-),12(-) & 1036.833         & 2.956 & 9(=),10(-),12(+) \\
            &           &      &           & 0.001    & 975.1    & 2.3288 & 10(-),11(+),12(-) & 989.5             & 2.202 & 9(+),11(+),12(=) & 954.866   & 10.616  & 9(-),10(-),12(-) & \textbf{991.3}   & 3.377 & 9(+),10(=),11(+) \\
            &           &      & 500000    & 0.1      & 1030.833 & 1.293  & 10(-),11(+),12(+) & \textbf{1041.633} & 1.251 & 9(+),11(+),12(+) & 1008.633  & 6.441   & 9(-),10(-),12(-) & 1035.233         & 2.641 & 9(+),10(-),12(+) \\
            &           &      &           & 0.001    & 967.666  & 3.418  & 10(-),11(+),12(-) & 985.9             & 2.748 & 9(+),11(+),12(=) & 947.233   & 10.932  & 9(-),10(-),12(-) & \textbf{988.866} & 4.145 & 9(+),10(=),11(+) \\
            &           & 2136 & 1500000   & 0.1      & 2035.066 & 2.92   & 10(-),11(+),12(+) & \textbf{2071.066} & 1.412 & 9(+),11(+),12(+) & 1963.4    & 9.844   & 9(-),10(-),12(-) & 2025.6           & 5.689 & 9(+),10(-),11(+) \\
            &           &      &           & 0.001    & 1925.3   & 3.671  & 10(-),11(+),12(-) & \textbf{1972.433} & 3.402 & 9(+),11(+),12(+) & 1850.633  & 13.345  & 9(-),10(-),12(-) & 1942.566         & 6.189 & 9(+),10(-),11(+) \\
            &           &      & 1000000   & 0.1      & 2026.966 & 3.341  & 10(-),11(+),12(+) & \textbf{2068.033} & 1.905 & 9(+),11(+),12(+) & 1956.3    & 9.987   & 9(-),10(-),12(-) & 2022.6           & 6.58  & 9(+),10(-),11(+) \\
            &           &      &           & 0.001    & 1914.7   & 3.831  & 10(-),11(+),12(-) & \textbf{1969.433} & 3.666 & 9(+),11(+),12(+) & 1839.766  & 13.313  & 9(-),10(-),12(-) & 1939.833         & 7.55  & 9(+),10(-),11(+) \\
            &           &      & 500000    & 0.1      & 2009.366 & 4.214  & 10(-),11(+),12(-) & \textbf{2063.166} & 2.852 & 9(+),11(+),12(+) & 1941.3    & 11.346  & 9(-),10(-),12(-) & 2016.5           & 5.942 & 9(+),10(-),11(+) \\
            &           &      &           & 0.001    & 1893.5   & 4.055  & 10(-),11(+),12(-) & \textbf{1960}     & 3.705 & 9(+),11(+),12(+) & 1821      & 14.61   & 9(-),10(-),12(-) & 1931.866         & 6.443 & 9(+),10(-),11(+) \\ \midrule
ca-CondaMat & Chernoff  & 146  & 1500000   & 0.1      & 139      & 0      & 10(=),11(=),12(=) & \textbf{139}      & 0     & 9(=),11(=),12(=) & 139       & 0       & 9(=),10(=),12(=) & \textbf{139}     & 0     & 9(=),10(=),11(=) \\
            &           &      &           & 0.001    & 137.1    & 1.445  & 10(=),11(=),12(=) & 138.2             & 1.326 & 9(=),11(=),12(=) & 136.9     & 1.374   & 9(=),10(=),12(=) & \textbf{138.3}   & 1.268 & 9(=),10(=),11(=) \\
            &           &      & 1000000   & 0.1      & 139      & 0      & 10(=),11(=),12(=) & \textbf{139}      & 0     & 9(=),11(=),12(=) & 139       & 0       & 9(=),10(=),12(=) & \textbf{139}     & 0     & 9(=),10(=),11(=) \\
            &           &      &           & 0.001    & 136.7    & 1.1268 & 10(=),11(=),12(=) & 137.8             & 1.469 & 9(=),11(=),12(=) & 136.5     & 1.118   & 9(=),10(=),12(=) & \textbf{138}     & 1.414 & 9(=),10(=),11(=) \\
            &           &      & 500000    & 0.1      & 139      & 0      & 10(=),11(+),12(=) & \textbf{139}      & 0     & 9(=),11(+),12(=) & 138.966   & 0.179   & 9(-),10(-),12(-) & \textbf{139}     & 0     & 9(=),10(=),11(+) \\
            &           &      &           & 0.001    & 136.4    & 1.019  & 10(=),11(=),12(=) & 136.8             & 1.326 & 9(=),11(=),12(=) & 136.266   & 0.928   & 9(=),10(=),12(=) & \textbf{137.1}   & 1.445 & 9(=),10(=),11(=) \\
            &           & 1068 & 1500000   & 0.1      & 1031.266 & 1.59   & 10(-),11(+),12(=) & \textbf{1039.366} & 1.139 & 9(+),11(+),12(+) & 1008.4    & 5.505   & 9(-),10(-),12(-) & 1033.966         & 2.994 & 9(=),10(-),11(+) \\
            &           &      &           & 0.001    & 1022.533 & 1.726  & 10(-),11(+),12(-) & \textbf{1031.533} & 0.956 & 9(+),11(+),12(+) & 1001.133  & 5.754   & 9(-),10(-),12(-) & 1026.966         & 2.575 & 9(+),10(-),11(+) \\
            &           &      & 1000000   & 0.1      & 1029.066 & 1.31   & 10(-),11(+),12(-) & \textbf{1038.166} & 1.097 & 9(+),11(+),12(+) & 1005.833  & 6.044   & 9(-),10(-),12(-) & 1033.266         & 3.203 & 9(+),10(-),11(+) \\
            &           &      &           & 0.001    & 1019.466 & 1.783  & 10(-),11(+),12(-) & \textbf{1030.366} & 1.425 & 9(+),11(+),12(=) & 997.466   & 6.463   & 9(-),10(-),12(-) & 1026.066         & 2.249 & 9(+),10(=),11(+) \\
            &           &      & 500000    & 0.1      & 1024.633 & 1.87   & 10(-),11(+),12(-) & \textbf{1036.133} & 0.956 & 9(+),11(+),12(+) & 1000.7    & 7.299   & 9(-),10(-),12(-) & 1031.766         & 3.921 & 9(+),10(-),11(+) \\
            &           &      &           & 0.001    & 1014.1   & 2.211  & 10(-),11(+),12(-) & \textbf{1027.833} & 1.507 & 9(+),11(+),12(=) & 992.5     & 7.069   & 9(-),10(-),12(-) & 1024.866         & 2.459 & 9(+),10(=),11(+) \\
            &           & 2136 & 1500000   & 0.1      & 2023.533 & 2.704  & 10(-),11(+),12(+) & \textbf{2062.166} & 1.881 & 9(+),11(+),12(+) & 1949.366  & 9.064   & 9(-),10(-),12(-) & 2019.5           & 6.687 & 9(+),10(-),11(+) \\
            &           &      &           & 0.001    & 2006.533 & 2.376  & 10(-),11(+),12(+) & \textbf{2041.2}   & 2.072 & 9(+),11(+),12(+) & 1932.4    & 10.694  & 9(-),10(-),12(-) & 2003.3           & 5.502 & 9(+),10(-),11(+) \\
            &           &      & 1000000   & 0.1      & 2015.366 & 3.219  & 10(-),11(+),12(=) & \textbf{2059.433} & 1.994 & 9(+),11(+),12(+) & 1942.266  & 9.051   & 9(-),10(-),12(-) & 2014.9           & 5.497 & 9(+),10(-),11(+) \\
            &           &      &           & 0.001    & 1997.5   & 3.232  & 10(-),11(+),12(-) & \textbf{2046.233} & 1.977 & 9(+),11(+),12(+) & 1924.566  & 11.632  & 9(-),10(-),12(-) & 2000.533         & 6.173 & 9(+),10(-),11(+) \\
            &           &      & 500000    & 0.1      & 1995.8   & 3.187  & 10(-),11(+),12(-) & \textbf{2052.833} & 2.296 & 9(+),11(+),12(+) & 1929.9    & 10.077  & 9(-),10(-),12(-) & 2008.766         & 6.897 & 9(+),10(-),11(+) \\
            &           &      &           & 0.001    & 1977.633 & 2.857  & 10(-),11(+),12(-) & \textbf{2037.933} & 2.555 & 9(+),11(+),12(+) & 1909.133  & 12.164  & 9(-),10(-),12(-) & 1993.9           & 7.449 & 9(+),10(-),11(+) \\ \bottomrule
\end{tabular}}
\end{table}

\begin{table}[ht]
\centering
\caption{Average number of trade-off solutions obtained by \gs and \sg in ca-CondaMat with IID weights}
\label{table:num_indiv}
\resizebox{0.75\textwidth}{3cm}{
\begin{tabular}{llllcccc}
\hline
            & $B$  & $t_{max}$ & $\alpha$ & \begin{tabular}[c]{@{}c@{}}\gs \\ $W_{cheb}$\end{tabular} & \begin{tabular}[c]{@{}c@{}}\sg \\ $W_{cheb}$\end{tabular} & \begin{tabular}[c]{@{}c@{}}\gs\\ $W_{chern}$\end{tabular} & \begin{tabular}[c]{@{}c@{}}\sg\\ $W_{chern}$\end{tabular} \\ \hline
ca-CondaMat & 146  & 1.5M   & 0.1   & 136                                                         & 136                                                             & 122                                                        & 122                                                            \\
&      &        & 0.001 & 70                                                          & 70                                                              & 107                                                        & 107                                                            \\
            &      & 1.0M   & 0.1   & 135                                                         & 136                                                             & 122                                                        & 122                                                            \\
            &      &        & 0.001 & 70                                                          & 70                                                              & 107                                                        & 107                                                            \\
            &      & 0.5M   & 0.1   & 136                                                         & 136                                                             & 122                                                        & 122                                                            \\
            &      &        & 0.001 & 70                                                          & 70                                                              & 107                                                        & 107                                                            \\
            & 1068 & 1.5M   & 0.1   & 927                                                         & 1039                                                            & 880                                                        & 998                                                            \\
            &      &        & 0.001 & 753                                                         & 808                                                             & 860                                                        & 948                                                            \\
            &      & 1.0M   & 0.1   & 883                                                         & 1040                                                            & 858                                                        & 997                                                            \\
            &      &        & 0.001 & 728                                                         & 809                                                             & 807                                                        & 949                                                            \\
            &      & 0.5M   & 0.1   & 720                                                         & 1034                                                            & 731                                                        & 997                                                            \\
            &      &        & 0.001 & 652                                                         & 809                                                             & 733                                                        & 948                                                            \\
            & 2136 & 1.5M   & 0.1   & 1160                                                        & 2066                                                            & 1174                                                       & 2024                                                           \\
            &      &        & 0.001 & 1174                                                        & 1752                                                            & 1217                                                       & 1960                                                           \\
            &      & 1.0M   & 0.1   & 968                                                         & 2073                                                            & 1006                                                       & 2022                                                           \\
            &      &        & 0.001 & 984                                                         & 1747                                                            & 989                                                        & 1947                                                           \\
            &      & 0.5M   & 0.1   & 707                                                         & 2050                                                            & 776                                                        & 1726                                                           \\
            &      &        & 0.001 & 748                                                         & 1735                                                            & 731                                                        & 1930                                                           \\ \hline
\end{tabular}
}
% \frank{Is this IID case?}
\end{table}
% \frank{You know how much smaller based on formulas.}

\begin{figure}[t]
    \centering
    \begin{subfigure}[b]{0.45\textwidth}
        \centering
        \includegraphics[scale=0.31]{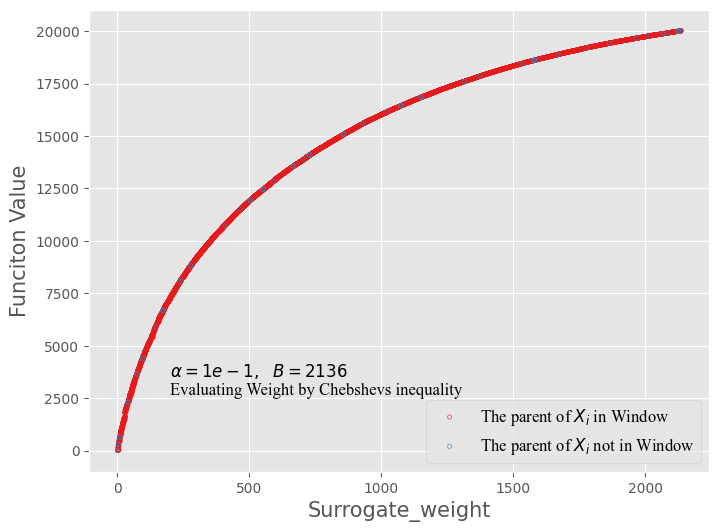}
        \caption{IID weights}
        \label{fig:condmat-1-opt}
    \end{subfigure}
    \hfill
    \begin{subfigure}[b]{0.45\textwidth}
        \centering
        \includegraphics[scale=0.31]{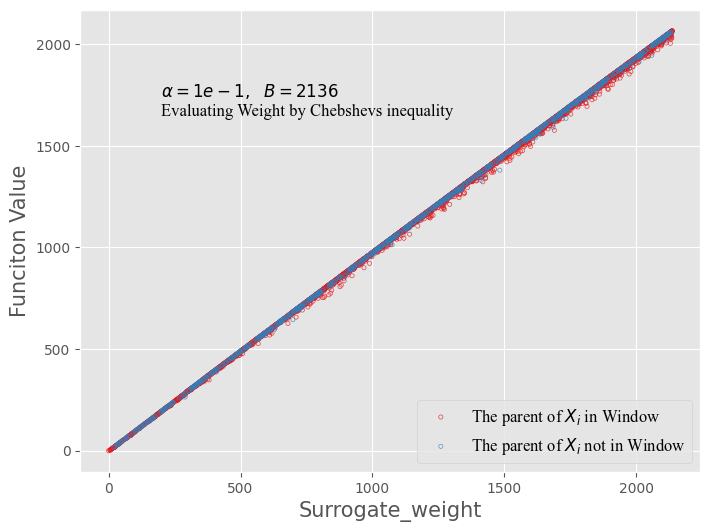}
        \caption{Uniform weights }
        \label{fig:condmat-2-opt}
    \end{subfigure}
    
    \caption{Optimization process of \sg for ca-CondaMat}
\end{figure}

\begin{figure}[t]
    \centering
    \begin{subfigure}[b]{0.45\textwidth}
        \centering
        \includegraphics[scale=0.31]{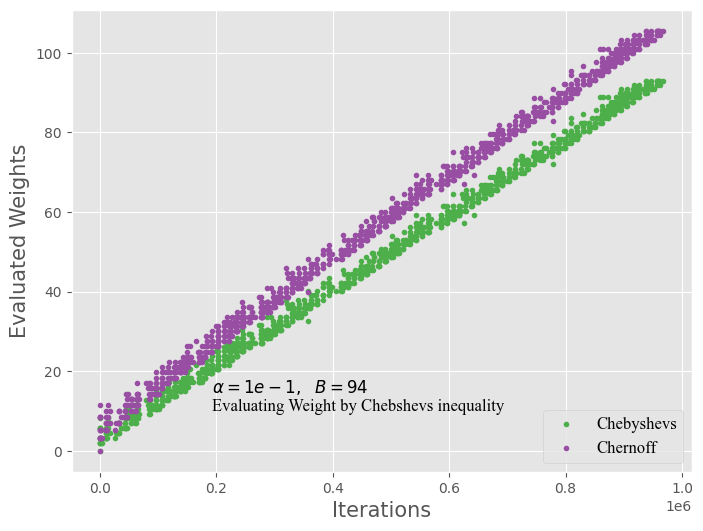}
        \caption{Chebshev's Inequality }
        \label{fig:cheb_vs_chef}
    \end{subfigure}
    \hfill
    \begin{subfigure}[b]{0.45\textwidth}
        \centering
        \includegraphics[scale=0.31]{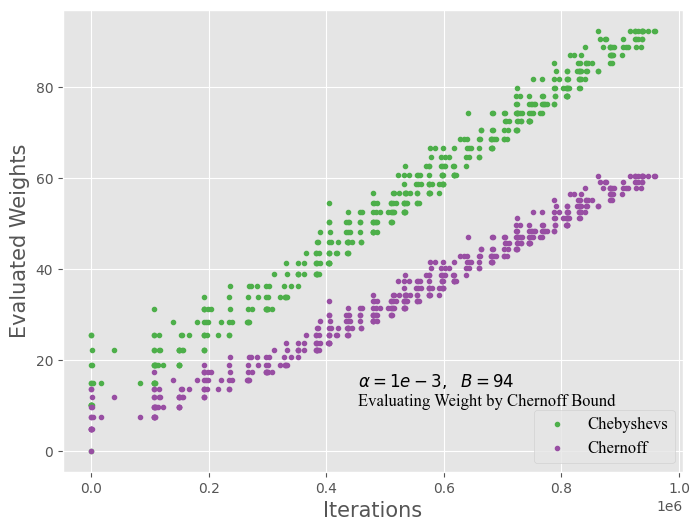}
        \caption{Chernoff bound}
        \label{fig:chef_vs_cheb}
    \end{subfigure}
    
    \caption{Different surrogate weights obtained during optimization by the \sg based on different surrogates in ca-CSphd with IID weights}
\end{figure}

Table~\ref{table:num_indiv} shows the average number of trade-off solutions obtained by the \gs and \sg using different surrogate functions on the problem based on the graph ca-CondaMat with IID weights.
Notably, the \sg produces a greater number of trade-off solutions in the final population than the \gs, particularly as the bound increases. 
Furthermore, the table reveals that the number of solutions in the population is consistently lower than $k$ (as defined in Section~\ref{sec:iid}, where $k = min\{n+1, \lfloor(B/a+1)\rfloor\}$).
When comparing with the deterministic setting, it is observed that the number of trade-off solutions generated by the algorithms in the IID weight setting decreases by almost 50\% when $\alpha = 0.001$, particularly when using the surrogate function based on one-sided Chebyshev's inequality.

\subsubsection{Visualization of \textit{SW-GSEMO}}

To focus on the \sg's performance in optimizing chance-constrained problems, 
Figures~\ref{fig:condmat-1-opt} and \ref{fig:condmat-2-opt} offer an illustrative example of the optimization process.
The figures illustrate the relationship between the surrogate weight and the function value of the solutions selected for the population, with different colors labeling solutions based on whether their parents were within the defined weight window.
Initially, an increase in function value corresponding to an increase in surrogate weight is observed.
It's noteworthy that the same surrogate weight might correspond to multiple distinct function values. 
According to the algorithm's domination scheme, among solutions with the same surrogate weight, all except the one with the highest function value are eliminated from the population.
Besides, the figures highlight that there are some periods where the \sg is unable to include any individuals within the window (particularly when the expected weights are uniform).
Despite these periods, the sliding window mechanism remains effective throughout the optimization process, aiding the algorithm in achieving satisfactory results. 
Additionally, those blue search points are also close to the Pareto front area and do not impact the final results.

Additionally, Figures~\ref{fig:cheb_vs_chef} and \ref{fig:chef_vs_cheb} describe the changes in surrogate weights of solutions across iterations.
A noticeable trend is that the sliding windows align the surrogate weights in a linear pattern,
where the surrogate weight generally increases with more iterations.
Figure~\ref{fig:cheb_vs_chef} indicates that the surrogate weight derived from the one-sided Chebyshev's inequality is lower than that from the Chernoff bound when $\alpha = 0.1$.
Conversely, Figure~\ref{fig:chef_vs_cheb} shows the opposite trend when $\alpha = 0.001$. 
Furthermore, due to the influence of the chance constraint,
the sliding-selection method, when using the surrogate weight, allows only one individual within the window under the IID wights setting
for the instances that are applied to the IID weights setting according to Figure~\ref{fig:chef_vs_cheb}, aligning with our analysis in Section~\ref{sec:iid}.
% This observation aligns with the pattern seen in Table~\ref{table:num_indiv}, where the algorithms using the one-sided Chebyshev's inequality approach yield more solutions at $\alpha = 0.1$ but fewer at $\alpha = 0.001$.

\section{Conclusion}
\label{sec:con}
% \frank{Conclusion rewritten, please check}
In this paper, we investigated the use of \sg on chance-constrained monotone submodular optimization problems with IID weights and uniform weights with the same dispersion. Surrogate functions based on Chebshev's inequality and Chernoff bound have been applied to evaluate the chance constraint.  
We showed theoretically that the \sg with the surrogate can reach the same approximation result in a more efficient way than the \gs that was studied in previous work. 
Furthermore, the algorithm is applied to the maximum coverage problem in the experiments and its results are compared with other multi-objective algorithms under variable instances constructed by different graphs.
The experiments demonstrated that the window defined in \sg is sliding in the weight interval during the optimization.  
Additionally, the obtained results show that the \sg with the surrogate based on one-sided Chebyshev's inequality performs better than the \gs and \nsga (with population sizes 20 and 100) among most of the instances when $\alpha$ is larger,
and the \sg using the Chernoff bound works best when $\alpha$ is smaller. 
For future work, it would be interesting to consider other generalized settings with different distributions and covariances as part of the chance-constrained formulation.

\begin{credits}
\subsubsection{\ackname} This work has been supported by the Australian Research Council (ARC) through grant FT200100536.

\end{credits}
\bibliographystyle{splncs04}
\bibliography{ref}

\newpage
\appendix

\section{Tables}
\label{app:a}

% \usepackage{booktabs}
\begin{table}[h]
\caption{Results for Maximum coverage problem with IID weights where the evaluation is based on Chebyshev's equality}
\label{table:iid_cheb}
\resizebox{\textwidth}{0.5\textwidth}{
\begin{tabular}{@{}llllllllllllllll@{}}
\toprule
            &      &           &          & \multicolumn{3}{l}{\gs (1)} & \multicolumn{3}{l}{\sg (2)} & \multicolumn{3}{l}{$\nsga_{20}$ (3)} & \multicolumn{3}{l}{$\nsga_{100}$ (4)} \\ \midrule
Graph       & $B$  & $t_{max}$ & $\alpha$ & Mean            & std        & stat             & Mean                 & std      & stat             & Mean           & std           & stat                & Mean              & std          & stat               \\
ca-CSphd    & 43   & 1500000   & 0.1      & \textbf{546}    & 0          & 2(=),3(+),4(+)   & \textbf{546}         & 0        & 1(=),3(+),4(+)   & 533.566        & 5.064         & 1(-),2(-),4(-)      & 545.233           & 0.76         & 1(-),2(-),3(+)     \\
&      &           & 0.001    & \textbf{238}    & 0          & 2(=),3(+),4(=)   & \textbf{238}         & 0        & 1(=),3(+),4(=)   & 237.6          & 0.611         & 1(-),2(-),4(-)      & \textbf{238}      & 0            & 1(=),2(=),3(-)     \\
            &      & 1000000   & 0.1      & 545.966         & 0.179      & 2(-),3(+),4(+)   & \textbf{546}         & 0        & 1(=),3(+),4(+)   & 528.933        & 6.196         & 1(-),2(-),4(-)      & 544.733           & 1.123        & 1(-),2(-),3(+)     \\
            &      &           & 0.001    & \textbf{238}    & 0          & 2(=),3(+),4(=)   & \textbf{238}         & 0        & 1(=),3(+),4(=)   & 236.933        & 0.727         & 1(-),2(-),4(-)      & \textbf{238}      & 0            & 1(=),2(=),3(+)     \\
            &      & 500000    & 0.1      & \textbf{543}    & 2.065      & 2(-),3(+),4(-)   & \textbf{546}         & 0        & 1(+),3(+),4(+)   & 517.266        & 8.35          & 1(-),2(-),4(-)      & 543.533           & 1.477        & 1(+),2(-),3(+)     \\
            &      &           & 0.001    & \textbf{238}    & 0          & 2(=),3(+),4(+)   & \textbf{238}         & 0        & 1(=),3(+),4(+)   & 235.966        & 1.538         & 1(-),2(-),4(-)      & 237.9             & 0.3          & 1(-),2(-),3(+)     \\
            & 94   & 1500000   & 0.1      & 880.993         & 0.727      & 2(-),3(+),4(+)   & \textbf{883}         & 0        & 1(-),3(+),4(+)   & 857.766        & 6.173         & 1(-),2(-),4(-)      & 870.133           & 2.753        & 1(-),2(-),3(+)     \\
            &      &           & 0.001    & \textbf{546}    & 0          & 2(=),3(+),4(+)   & \textbf{546}         & 0        & 1(=),3(+),4(+)   & 531.433        & 5.613         & 1(-),2(-),4(-)      & 545.5             & 0.921        & 1(-),2(-),3(+)     \\
            &      & 1000000   & 0.1      & 875.966         & 2.469      & 2(-),3(+),4(+)   & \textbf{883}         & 0        & 1(+),3(+),4(+)   & 845.6          & 8.89          & 1(-),2(-),4(-)      & 865.9             & 3.515        & 1(-),2(-),3(+)     \\
            &      &           & 0.001    & 545.966         & 0.179      & 2(-),3(+),4(+)   & \textbf{546}         & 0        & 1(+),3(+),4(+)   & 523.233        & 6.338         & 1(-),2(-),4(-)      & 544.833           & 0.897        & 1(-),2(-),3(+)     \\
            &      & 500000    & 0.1      & 848.533         & 4.1        & 2(-),3(+),4(-)   & \textbf{882.766}     & 0.422    & 1(+),3(+),4(+)   & 821.733        & 10.168        & 1(-),2(-),4(-)      & 858.4             & 3.878        & 1(+),2(-),3(+)     \\
            &      &           & 0.001    & 543.1           & 1.738      & 2(-),3(+),4(-)   & \textbf{546}         & 0        & 1(+),3(+),4(+)   & 511.333        & 8.117         & 1(-),2(-),4(-)      & 543.366           & 1.471        & 1(+),2(-),3(+)     \\
            & 188  & 1500000   & 0.1      & 1234.066        & 2.128      & 2(-),3(+),4(+)   & \textbf{1243.233}    & 0.76     & 1(+),3(+),4(+)   & 1225           & 3.941         & 1(-),2(-),4(-)      & 1220.433          & 3.602        & 1(-),2(-),3(+)     \\
            &      &           & 0.001    & 941.366         & 0.572      & 2(-),3(+),4(+)   & \textbf{942.933}     & 0.359    & 1(+),3(+),4(+)   & 919.133        & 5.01          & 1(-),2(-),4(-)      & 925.133           & 3.116        & 1(-),2(-),3(+)     \\
            &      & 1000000   & 0.1      & 1213.96         & 3.281      & 2(-),3(+),4(-)   & \textbf{1243.23}     & 0.715    & 1(+),3(+),4(+)   & 1210.866       & 6.781         & 1(-),2(-),4(-)      & 1214.133          & 3.685        & 1(-),2(-),3(+)     \\
            &      &           & 0.001    & 934.9           & 1.795      & 2(-),3(+),4(+)   & \textbf{943}         & 0        & 1(+),3(+),4(+)   & 906            & 7.478         & 1(-),2(-),4(-)      & 922.633           & 4.094        & 1(-),2(-),3(+)     \\
            &      & 500000    & 0.1      & 1154.633        & 6.332      & 2(-),3(+),4(-)   & \textbf{1243.466}    & 0.618    & 1(+),3(+),4(+)   & 1179.3         & 8.509         & 1(-),2(-),4(-)      & 1200.766          & 5.308        & 1(+),2(-),3(+)     \\
            &      &           & 0.001    & 903.7           & 5.386      & 2(-),3(+),4(-)   & \textbf{942.933}     & 0.249    & 1(+),3(+),4(+)   & 876            & 12.492        & 1(-),2(-),4(-)      & 915.5             & 4.883        & 1(+),2(-),3(+)     \\ \midrule
ca-GrQc     & 64   & 1500000   & 0.1      & 1403.933        & 7.54       & 2(-),3(+),4(+)   & \textbf{1432.333}    & 1.534    & 1(+),3(+),4(+)   & 1305.666       & 20.426        & 1(-),2(-),4(-)      & 1395.9            & 11.527       & 1(-),2(-),3(+)     \\
 &      &           & 0.001    & 754.266         & 3.14       & 2(+),3(+),4(-)   & \textbf{756.933}     & 0.249    & 1(-),3(+),4(-)   & 724.7          & 9.212         & 1(-),2(-),4(-)      & 754.7             & 2.368        & 1(+),2(+),3(+)     \\
            &      & 1000000   & 0.1      & 1387.5          & 7.428      & 2(-),3(+),4(+)   & \textbf{1431.733}    & 2.644    & 1(+),3(+),4(+)   & 1289.6         & 19.338        & 1(-),2(-),4(-)      & 1386.26           & 12.465       & 1(-),2(-),3(+)     \\
            &      &           & 0.001    & 746.033         & 7.323      & 2(-),3(+),4(-)   & \textbf{756.966}     & 0.179    & 1(+),3(+),4(-)   & 717.066        & 13.132        & 1(-),2(-),4(-)      & 753.033           & 2.96         & 1(+),2(+),3(+)     \\
            &      & 500000    & 0.1      & 1332.466        & 12.831     & 2(-),3(+),4(-)   & \textbf{1427.633}    & 4.693    & 1(+),3(+),4(+)   & 1233.9         & 19.618        & 1(-),2(-),4(-)      & 1369.066          & 14.104       & 1(+),2(-),4(+)     \\
            &      &           & 0.001    & 733.2           & 7.93       & 2(-),3(+),4(-)   & \textbf{754.966}     & 4.118    & 1(+),3(+),4(+)   & 698.9          & 15.788        & 1(-),2(-),4(-)      & 748.133           & 7.214        & 1(+),2(-),3(+)     \\
            & 207  & 1500000   & 0.1      & 2516.2          & 13.929     & 2(-),3(+),4(-)   & \textbf{2694.6}      & 4.506    & 1(+),3(+),4(+)   & 2428.466       & 21.846        & 1(-),2(-),4(-)      & 2551.866          & 14.176       & 1(+),2(-),3(+)     \\
            &      &           & 0.001    & 1974.466        & 10.206     & 2(-),3(+),4(+)   & \textbf{2053.9}      & 3.279    & 1(+),3(+),4(+)   & 1837.066       & 23.589        & 1(-),2(-),4(-)      & 1968.566          & 16.111       & 1(-),2(-),3(+)     \\
            &      & 1000000   & 0.1      & 2434.9          & 13.55      & 2(-),3(+),4(+)   & \textbf{2691.466}    & 4.462    & 1(+),3(+),4(+)   & 2363.966       & 25.356        & 1(-),2(-),4(-)      & 2535.766          & 16.823       & 1(+),2(-),3(+)     \\
            &      &           & 0.001    & 1930.133        & 10.901     & 2(-),3(+),4(-)   & \textbf{2054.833}    & 3.652    & 1(+),3(+),4(+)   & 1793.833       & 27.668        & 1(-),2(-),4(-)      & 1958.7            & 18.018       & 1(-),2(-),3(+)     \\
            &      & 500000    & 0.1      & 2288.033        & 14.549     & 2(-),3(+),4(-)   & \textbf{2688}        & 6.957    & 1(+),3(+),4(+)   & 2230.6         & 26.925        & 1(-),2(-),4(-)      & 2487.6            & 23.171       & 1(-),2(-),3(+)     \\
            &      &           & 0.001    & 1821.933        & 13.985     & 2(-),3(+),4(-)   & \textbf{2049.1}      & 5.081    & 1(+),3(+),4(+)   & 1704.366       & 30.625        & 1(-),2(-),4(-)      & 1925.133          & 16.202       & 1(-),2(-),3(+)     \\
            & 415  & 1500000   & 0.1      & 3205.966        & 12.605     & 2(-),3(-),4(-)   & \textbf{3556.866}    & 3.518    & 1(+),3(+),4(+)   & 3270.966       & 23.345        & 1(+),2(-),4(-)      & 3296.6            & 14.63        & 1(-),2(-),3(+)     \\
            &      &           & 0.001    & 2822.733        & 13.053     & 2(-),3(+),4(-)   & \textbf{3078.9}      & 4.407    & 1(+),3(+),4(+)   & 2777.2         & 17.158        & 1(-),2(-),4(-)      & 2865.933          & 16.29        & 1(-),2(-),3(+)     \\
            &      & 1000000   & 0.1      & 3105            & 11.195     & 2(-),3(-),4(-)   & \textbf{3555.3}      & 4.267    & 1(+),3(+),4(+)   & 3179.033       & 22.15         & 1(+),2(-),4(-)      & 3264.433          & 14.718       & 1(-),2(-),3(+)     \\
            &      &           & 0.001    & 2734.76         & 13.934     & 2(-),3(+),4(-)   & \textbf{3076.2}      & 4.867    & 1(+),3(+),4(+)   & 2688.73        & 28.249        & 1(-),2(-),4(-)      & 2832.9            & 16.933       & 1(-),2(-),3(+)     \\
            &      & 500000    & 0.1      & 2921.366        & 18.076     & 2(-),3(-),4(-)   & \textbf{3548.8}      & 4.969    & 1(+),3(+),4(+)   & 3000.266       & 25.241        & 1(+),2(-),4(-)      & 3198.266          & 17.804       & 1(-),2(-),3(+)     \\
            &      &           & 0.001    & 2569.566        & 17.657     & 2(-),3(+),4(-)   & \textbf{3070.5}      & 6.206    & 1(+),3(+),4(+)   & 2525.533       & 30.721        & 1(-),2(-),4(-)      & 2774.433          & 17.44        & 1(-),2(-),3(+)     \\ \midrule
ca-CondaMat & 146  & 1500000   & 0.1      & 5588            & 47.265     & 2(-),3(+),4(-)   & \textbf{6790.966}    & 12.335   & 1(+),3(+),4(+)   & 5187.966       & 95.708        & 1(-),2(-),4(-)      & 6330.3            & 36.893       & 1(+),2(-),3(+)     \\
 &      &           & 0.001    & 4153.733        & 43.338     & 2(-),3(+),4(-)   & \textbf{4748.166}    & 7.585    & 1(+),3(+),4(+)   & 3802.733       & 75.95         & 1(-),2(-),4(-)      & 4531.83           & 33.061       & 1(+),2(-),3(+)     \\
            &      & 1000000   & 0.1      & 5500.73         & 47.67      & 2(-),3(+),4(-)   & \textbf{6771.366}    & 16.3     & 1(+),3(+),4(+)   & 4945.9         & 113.293       & 1(-),2(-),4(-)      & 6281              & 46.322       & 1(+),2(-),3(+)     \\
            &      &           & 0.001    & 3957.5          & 40.782     & 2(-),3(+),4(-)   & \textbf{4736.133}    & 10.375   & 1(+),3(+),4(+)   & 3709.06        & 96.656        & 1(-),2(-),4(-)      & 4496.7            & 28.765       & 1(+),2(-),3(+)     \\
            &      & 500000    & 0.1      & 4818.53         & 47.71      & 2(-),3(+),4(-)   & \textbf{6708.366}    & 27.316   & 1(+),3(+),4(+)   & 4685.333       & 114.8373      & 1(-),2(-),4(-)      & 6115.233          & 54.393       & 1(+),2(-),3(+)     \\
            &      &           & 0.001    & 3648.3          & 65.33      & 2(-),3(+),4(-)   & \textbf{4581.633}    & 36.88    & 1(+),3(+),4(+)   & 3469.766       & 73.953        & 1(-),2(-),4(-)      & 4395.566          & 43.076       & 1(+),2(-),3(+)     \\
            & 1068 & 1500000   & 0.1      & 11787.533       & 57.133     & 2(-),3(+),4(-)   & \textbf{16650.933}   & 14.163   & 1(+),3(+),4(+)   & 12284.866      & 142.248       & 1(-),2(-),4(-)      & 13394.133         & 75.448       & 1(+),2(-),3(+)     \\
            &      &           & 0.001    & 10893.9         & 61.683     & 2(-),3(+),4(-)   & \textbf{15217.3}     & 14.45    & 1(+),3(+),4(+)   & 10950.3        & 122.3         & 1(-),2(-),4(-)      & 12241.966         & 72.716       & 1(+),2(-),3(+)     \\
            &      & 1000000   & 0.1      & 11194.23        & 72.488     & 2(-),3(-),4(-)   & \textbf{16573.6}     & 19.608   & 1(+),3(+),4(+)   & 11623.93       & 119.467       & 1(+),2(-),4(-)      & 13164.2           & 60.07        & 1(+),2(-),3(+)     \\
            &      &           & 0.001    & 10364.53        & 58.95      & 2(-),3(+),4(-)   & \textbf{15145.666}   & 14.485   & 1(+),3(+),4(+)   & 9987.13        & 123.936       & 1(-),2(-),4(-)      & 12040.166         & 99.78        & 1(+),2(-),3(+)     \\
            &      & 500000    & 0.1      & 9474.733        & 111.851    & 2(-),3(-),4(-)   & \textbf{16368.6}     & 27.005   & 1(+),3(+),4(+)   & 10729.133      & 180.573       & 1(+),2(-),4(-)      & 12708.533         & 82.447       & 1(+),2(-),3(+)     \\
            &      &           & 0.001    & 9344.733        & 84.931     & 2(-),3(-),4(-)   & \textbf{14947.6}     & 22.553   & 1(+),3(+),4(+)   & 9502.2         & 142.003       & 1(+),2(-),4(-)      & 11581.066         & 84.55        & 1(+),2(-),3(+)     \\
            & 2136 & 1500000   & 0.1      & 12749.533       & 93.378     & 2(-),3(-),4(-)   & \textbf{20078.833}   & 11.066   & 1(+),3(+),4(+)   & 16361.766      & 67.76         & 1(+),2(-),4(-)      & 16243.166         & 68.601       & 1(+),2(-),3(+)     \\
            &      &           & 0.001    & 12730.3         & 91.025     & 2(-),3(-),4(-)   & \textbf{19327.433}   & 11.221   & 1(+),3(+),4(+)   & 15224.733      & 66.313        & 1(+),2(-),4(-)      & 15402.133         & 96.35        & 1(+),2(-),3(+)     \\
            &      & 1000000   & 0.1      & 11520.966       & 129.843    & 2(-),3(-),4(-)   & \textbf{20016.2}     & 16.172   & 1(+),3(+),4(+)   & 15696.5        & 116.216       & 1(+),2(-),4(-)      & 16020.633         & 88.739       & 1(+),2(-),3(+)     \\
            &      &           & 0.001    & 11500.633       & 114.875    & 2(-),3(-),4(-)   & \textbf{19245.233}   & 15.532   & 1(+),3(+),4(+)   & 14544.63       & 98.014        & 1(+),2(-),4(-)      & 15155.4           & 74.248       & 1(+),2(-),3(+)     \\
            &      & 500000    & 0.1      & 9488.333        & 90.631     & 2(-),3(-),4(-)   & \textbf{19822.566}   & 20.619   & 1(+),3(+),4(+)   & 14592.466      & 104.587       & 1(+),2(-),4(-)      & 14734.2           & 230.948      & 1(+),2(-),3(+)     \\
            &      &           & 0.001    & 9483.6          & 88.03      & 2(-),3(-),4(-)   & \textbf{19030.533}   & 21.451   & 1(+),3(+),4(+)   & 13456.133      & 89.165        & 1(+),2(-),4(-)      & 14399.233         & 152.192      & 1(+),2(-),3(+)     \\ \bottomrule
\end{tabular}
}
\end{table}

\begin{table}[t]
\caption{Results for Maximum coverage problem with IID weights where the evaluation is based on Chernoff Bound}
\label{table:iid_chef}
\resizebox{\textwidth}{0.5\textwidth}{
\begin{tabular}{@{}llllllllllllllll@{}}
\toprule
            &      &           &          & \multicolumn{3}{l}{\gs   (5)} & \multicolumn{3}{l}{\sg (6)}   & \multicolumn{3}{l}{$\nsga_{20}$ (7)}    & \multicolumn{3}{l}{$\nsga_{100}$ (8)} \\ \midrule
Graph       & $B$  & $t_{max}$ & $\alpha$ & Mean           & std       & stat            & Mean               & std    & stat           & Mean         & std     & stat           & Mean       & std     & stat           \\
ca-CSphd    & 43   & 1500000   & 0.1      & \textbf{478}   & 0         & 6(=),7(+),8(+)  & \textbf{478}       & 0      & 5(=),7(+),8(+) & 468.333      & 4.853   & 5(-),6(-),8(-) & 477.766    & 0.495   & 5(-),6(-),7(+) \\
            &      &           & 0.001    & \textbf{413}   & 0         & 6(=),7(+),8(+)  & \textbf{413}       & 0      & 5(=),7(+),8(+) & 405.4        & 3.903   & 5(-),6(-),8(-) & 412.9      & 0.3     & 5(-),6(-),7(+) \\
            &      & 1000000   & 0.1      & \textbf{478}   & 0         & 6(=),7(+),8(+)  & \textbf{478}       & 0      & 5(=),7(+),8(+) & \textbf{463} & 5.403   & 5(-),6(-),8(-) & 477.7      & 0.781   & 5(-),6(-),7(+) \\
            &      &           & 0.001    & \textbf{413}   & 0         & 6(=),7(+),8(+)  & \textbf{413}       & 0      & 5(=),7(+),8(+) & 402.6        & 4.24    & 5(-),6(-),8(-) & 412.633    & 0.481   & 5(-),6(-),7(+) \\
            &      & 500000    & 0.1      & 477.966        & 0.179     & 6(-),7(+),8(+)  & \textbf{477.966}   & 0.179  & 5(+),7(+),8(+) & 453.566      & 7.269   & 5(-),6(-),8(-) & 477.166    & 0.968   & 5(-),6(-),7(+) \\
            &      &           & 0.001    & 412.9          & 0.3       & 6(-),7(+),8(+)  & \textbf{413}       & 0      & 5(+),7(+),8(+) & \textbf{396} & 5.196   & 5(-),6(-),8(-) & 412.066    & 0.771   & 5(-),6(-),7(+) \\
            & 94   & 1500000   & 0.1      & 817.7          & 0.525     & 6(-),7(+),8(+)  & \textbf{818}       & 0      & 5(+),7(+),8(+) & 796.7        & 5.484   & 5(-),6(-),8(-) & 808.933    & 3.14    & 5(-),6(-),7(+) \\
            &      &           & 0.001    & 749.633        & 0         & 6(-),7(+),8(+)  & \textbf{750}       & 0      & 5(+),7(+),8(+) & 729.066      & 5.938   & 5(-),6(-),8(-) & 743        & 2.081   & 5(-),6(-),7(+) \\
            &      & 1000000   & 0.1      & 817.5          & 0.67      & 6(-),7(+),8(+)  & \textbf{817.966}   & 0.179  & 5(+),7(+),8(+) & 785.766      & 7.218   & 5(-),6(-),8(-) & 806.9      & 3.703   & 5(-),6(-),7(+) \\
            &      &           & 0.001    & 749.633        & 0.546     & 6(-),7(+),8(+)  & \textbf{750}       & 0      & 5(+),7(+),8(+) & 718.466      & 7.658   & 5(-),6(-),8(-) & 740.5      & 2.86    & 5(-),6(-),7(+) \\
            &      & 500000    & 0.1      & 809.8          & 1.956     & 6(-),7(+),8(+)  & \textbf{818}       & 0      & 5(+),7(+),8(+) & 764.733      & 9.337   & 5(-),6(-),8(-) & 800.766    & 3.48    & 5(-),6(-),7(+) \\
            &      &           & 0.001    & 744.633        & 2.057     & 6(-),7(+),8(+)  & \textbf{749.96}    & 0.179  & 5(+),7(+),8(+) & \textbf{695} & 9.855   & 5(-),6(-),8(-) & 735.666    & 4.307   & 5(-),6(-),7(+) \\
            & 188  & 1500000   & 0.1      & 1181.733       & 2.542     & 6(-),7(+),8(+)  & \textbf{1192.466}  & 0.618  & 5(+),7(+),8(+) & 1166.966     & 5.003   & 5(-),6(-),8(-) & 1166.966   & 3.772   & 5(-),6(-),7(+) \\
            &      &           & 0.001    & 1120.6         & 1.89      & 6(-),7(+),8(+)  & \textbf{1128}      & 0      & 5(+),7(+),8(+) & 1103.5       & 4.595   & 5(-),6(-),8(-) & 1105.8     & 3.572   & 5(-),6(-),7(+) \\
            &      & 1000000   & 0.1      & 1181.233       & 2.216     & 6(-),7(+),8(+)  & \textbf{1192.3}    & 0.69   & 5(+),7(+),8(+) & 1153.766     & 5.696   & 5(-),6(-),8(-) & 1160.033   & 3.745   & 5(-),6(-),7(+) \\
            &      &           & 0.001    & 1120.633       & 1.957     & 6(-),7(+),8(+)  & \textbf{1127.966}  & 0.179  & 5(+),7(+),8(+) & 1090.766     & 6.907   & 5(-),6(-),8(-) & 1099.6     & 2.961   & 5(-),6(-),7(+) \\
            &      & 500000    & 0.1      & 1143.833       & 5.865     & 6(-),7(+),8(-)  & \textbf{1192.366}  & 0.546  & 5(+),7(+),8(+) & 1122.1       & 8.904   & 5(-),6(-),8(-) & 1147.166   & 4.993   & 5(+),6(-),7(+) \\
            &      &           & 0.001    & 1089.333       & 4.101     & 6(-),7(+),8(+)  & \textbf{1127.9}    & 0.3    & 5(+),7(+),8(+) & 1057.266     & 10.327  & 5(-),6(-),8(-) & 1089.166   & 4.719   & 5(-),6(-),7(+) \\ \midrule
ca-GrQc     & 64   & 1500000   & 0.1      & 1275.6         & 6.311     & 6(-),7(+),8(+)  & \textbf{1294.433}  & 2.076  & 5(+),7(+),8(+) & 1189.966     & 17.809  & 5(-),6(-),8(-) & 1265.633   & 6.695   & 5(-),6(-),7(+) \\
            &      &           & 0.001    & 1125.76        & 5.69      & 6(-),7(+),8(+)  & \textbf{1145.433}  & 0.955  & 5(+),7(+),8(+) & 1054.733     & 18.446  & 5(-),6(-),8(-) & 1125.2     & 7.409   & 5(-),6(-),7(+) \\
            &      & 1000000   & 0.1      & 1275.233       & 7.548     & 6(-),7(+),8(+)  & \textbf{1294.1}    & 2.399  & 5(+),7(+),8(+) & 1168.6       & 16.318  & 5(-),6(-),8(-) & 1259.266   & 10.478  & 5(-),6(-),7(+) \\
            &      &           & 0.001    & 1127.3         & 7.528     & 6(-),7(+),8(+)  & \textbf{1144.866}  & 1.431  & 5(+),7(+),8(+) & 1031.1       & 19.291  & 5(-),6(-),8(-) & 1121.933   & 8.156   & 5(-),6(-),7(+) \\
            &      & 500000    & 0.1      & 1249.5         & 11.242    & 6(-),7(+),8(+)  & \textbf{1291.6}    & 4.095  & 5(+),7(+),8(+) & 1126.333     & 21.235  & 5(-),6(-),8(-) & 1246       & 13.147  & 5(-),6(-),7(+) \\
            &      &           & 0.001    & 1105.533       & 8.815     & 6(-),7(+),8(+)  & \textbf{1143.266}  & 2.112  & 5(+),7(+),8(+) & 992.333      & 19.777  & 5(-),6(-),8(-) & 1112.933   & 9.44    & 5(-),6(-),7(+) \\
            & 207  & 1500000   & 0.1      & 2426.533       & 10.704    & 6(-),7(+),8(-)  & \textbf{2585.5}    & 4.514  & 5(+),7(+),8(+) & 2321.133     & 29.9    & 5(-),6(-),8(-) & 2454.966   & 16.664  & 5(+),6(-),7(+) \\
            &      &           & 0.001    & 2316.666       & 9.133     & 6(-),7(+),8(-)  & \textbf{2452.633}  & 4.956  & 5(+),7(+),8(+) & 2205.166     & 19.834  & 5(-),6(-),8(-) & 2328.2     & 19.436  & 5(+),6(-),7(+) \\
            &      & 1000000   & 0.1      & 2425.333       & 13.196    & 6(-),7(+),8(-)  & \textbf{2582}      & 5.899  & 5(+),7(+),8(+) & 2258.666     & 25.941  & 5(-),6(-),8(-) & 2437.7     & 13.256  & 5(+),6(-),7(+) \\
            &      &           & 0.001    & 2315.7666      & 10.4      & 6(-),7(+),8(+)  & \textbf{2450.066}  & 5.585  & 5(+),7(+),8(+) & 2151.4       & 34.358  & 5(-),6(-),8(-) & 2313.766   & 17.392  & 5(-),6(-),7(+) \\
            &      & 500000    & 0.1      & 2294.366       & 13.212    & 6(-),7(+),8(-)  & \textbf{2577.566}  & 6.751  & 5(+),7(+),8(+) & 2133.1       & 32.639  & 5(-),6(-),8(-) & 2396.5     & 12.241  & 5(+),6(-),7(+) \\
            &      &           & 0.001    & 2194.233       & 13.142    & 6(-),7(+),8(+)  & \textbf{2446.233}  & 5.69   & 5(+),7(+),8(+) & 2028.766     & 34.202  & 5(-),6(-),8(-) & 2276.9     & 16.44   & 5(-),6(-),7(+) \\
            & 415  & 1500000   & 0.1      & 3142.9         & 10.746    & 6(-),7(-),8(-)  & \textbf{3479.133}  & 3.77   & 5(+),7(+),8(+) & 3182.9       & 20.115  & 5(+),6(-),8(-) & 3227.633   & 16.15   & 5(+),6(-),7(+) \\
            &      &           & 0.001    & 3069.366       & 15.047    & 6(-),7(-),8(-)  & \textbf{3394.733}  & 4.17   & 5(+),7(+),8(+) & 3080.033     & 17.564  & 5(+),6(-),8(-) & 3140.2     & 17.457  & 5(+),6(-),7(+) \\
            &      & 1000000   & 0.1      & 3133.366       & 12.768    & 6(-),7(+),8(-)  & \textbf{3477.833}  & 4.568  & 5(+),7(+),8(+) & 3098.233     & 27.465  & 5(-),6(-),8(-) & 3194.066   & 18.145  & 5(+),6(-),7(+) \\
            &      &           & 0.001    & 3055.366       & 13.496    & 6(-),7(+),8(-)  & \textbf{3391.766}  & 4.63   & 5(+),7(+),8(+) & 2999.766     & 35.903  & 5(-),6(-),8(-) & 3105.633   & 17.516  & 5(+),6(-),7(+) \\
            &      & 500000    & 0.1      & 2949.733       & 14.955    & 6(-),7(+),8(-)  & \textbf{3470.666}  & 4.763  & 5(+),7(+),8(+) & 2920.233     & 35.273  & 5(-),6(-),8(-) & 3127.566   & 16.562  & 5(+),6(-),7(+) \\
            &      &           & 0.001    & 2874.9         & 14.767    & 6(-),7(+),8(-)  & \textbf{3382.833}  & 4.442  & 5(+),7(+),8(+) & 2836         & 32.535  & 5(-),6(-),8(-) & 3039.433   & 17.562  & 5(+),6(-),7(+) \\ \midrule
ca-CondaMat & 146  & 1500000   & 0.1      & 5321.1333      & 40.291    & 6(-),7(+),8(-)  & \textbf{6424.2}    & 10.403 & 5(+),7(+),8(+) & 4931.833     & 82.145  & 5(-),6(-),8(-) & 6018.066   & 41.946  & 5(+),6(-),7(+) \\
            &      &           & 0.001    & 5027.166       & 49.31     & 6(-),7(+),8(-)  & \textbf{5994.833}  & 10.96  & 5(+),7(+),8(+) & 4622.866     & 96.07   & 5(-),6(-),8(-) & 5652.833   & 41.121  & 5(+),6(-),7(+) \\
            &      & 1000000   & 0.1      & 5059.6         & 39.678    & 6(-),7(+),8(-)  & \textbf{6397.966}  & 14.943 & 5(+),7(+),8(+) & 4755.133     & 73.411  & 5(-),6(-),8(-) & 5954.966   & 50.251  & 5(+),6(-),7(+) \\
            &      &           & 0.001    & 4784.766       & 54.93     & 6(-),7(+),8(-)  & \textbf{5979.9}    & 16.912 & 5(+),7(+),8(+) & 4441.5       & 126.327 & 5(-),6(-),8(-) & 5582.666   & 38.694  & 5(+),6(-),7(+) \\
            &      & 500000    & 0.1      & 4625.4         & 60.563    & 6(-),7(+),8(-)  & \textbf{6328.2}    & 31.971 & 5(+),7(+),8(+) & 4443.2       & 84.517  & 5(-),6(-),8(-) & 5787.266   & 63.911  & 5(+),6(-),7(+) \\
            &      &           & 0.001    & 4344.33        & 59.972    & 6(-),7(+),8(-)  & \textbf{5898.133}  & 22.47  & 5(+),7(+),8(+) & 4170         & 103.826 & 5(-),6(-),8(-) & 5437.7     & 43.076  & 5(+),6(-),7(+) \\
            & 1068 & 1500000   & 0.1      & 11632.833      & 52.573    & 6(-),7(-),8(-)  & \textbf{16650.933} & 14.163 & 5(+),7(+),8(+) & 12054.233    & 126.656 & 5(+),6(-),8(-) & 13206.26   & 65.605  & 5(+),6(-),7(+) \\
            &      &           & 0.001    & 11464.966      & 61.738    & 6(-),7(-),8(-)  & \textbf{15217.3}   & 14.45  & 5(+),7(+),8(+) & 11783.9      & 95.969  & 5(+),6(-),8(-) & 12974.86   & 85.587  & 5(+),6(-),7(+) \\
            &      & 1000000   & 0.1      & 11059.2        & 73.482    & 6(-),7(-),8(-)  & \textbf{16343.833} & 16.806 & 5(+),7(+),8(+) & 11441.2      & 145.704 & 5(+),6(-),8(-) & 12961.266  & 85.324  & 5(+),6(-),7(+) \\
            &      &           & 0.001    & 10914.833      & 76.11     & 6(-),7(-),8(-)  & \textbf{16052.866} & 19.687 & 5(+),7(+),8(+) & 11208.1      & 91.525  & 5(+),6(-),8(-) & 12779.033  & 59.812  & 5(+),6(-),7(+) \\
            &      & 500000    & 0.1      & 9482.966       & 92.698    & 6(-),7(-),8(-)  & \textbf{16129.133} & 27.284 & 5(+),7(+),8(+) & 10487.566    & 128.411 & 5(+),6(-),8(-) & 12489.3    & 88.034  & 5(+),6(-),7(+) \\
            &      &           & 0.001    & 9466.433       & 113.403   & 6(-),7(-),8(-)  & \textbf{15840.433} & 26.94  & 5(+),7(+),8(+) & 10254.133    & 130.19  & 5(+),6(-),8(-) & 12309.4    & 81.884  & 5(+),6(-),7(+) \\
            & 2136 & 1500000   & 0.1      & 12719.533      & 106.011   & 6(-),7(-),8(-)  & \textbf{19955.3}   & 10.312 & 5(+),7(+),8(+) & 16187.466    & 93.583  & 5(+),6(-),8(-) & 16130.1    & 80.711  & 5(+),6(-),7(+) \\
            &      &           & 0.001    & 12701.7        & 90.339    & 6(-),7(-),8(-)  & \textbf{19813.8}   & 14.041 & 5(+),7(+),8(+) & 16006.8      & 109.226 & 5(+),6(-),8(-) & 15964.933  & 75.003  & 5(+),6(-),7(+) \\
            &      & 1000000   & 0.1      & 11481.066      & 98.619    & 6(-),7(-),8(-)  & \textbf{19891.466} & 13.197 & 5(+),7(+),8(+) & 15519.3      & 102.403 & 5(+),6(-),8(-) & 15812.433  & 93.623  & 5(+),6(-),7(+) \\
            &      &           & 0.001    & 11458.1        & 116.396   & 6(-),7(-),8(-)  & \textbf{19750.366} & 12.084 & 5(+),7(+),8(+) & 15323.266    & 103.988 & 5(+),6(-),8(-) & 15684.533  & 83.137  & 5(+),6(-),7(+) \\
            &      & 500000    & 0.1      & 9451.033       & 105.529   & 6(-),7(-),8(-)  & \textbf{19697.033} & 17.995 & 5(+),7(+),8(+) & 14435.233    & 101.01  & 5(+),6(-),8(-) & 14650.033  & 289.611 & 5(+),6(-),7(+) \\
            &      &           & 0.001    & 9460.666       & 116.783   & 6(-),7(-),8(-)  & \textbf{19540.266} & 18.77  & 5(+),7(+),8(+) & 14218.8      & 129.106 & 5(+),6(-),8(-) & 14586.566  & 214.181 & 5(+),6(-),7(+) \\ \bottomrule
\end{tabular}
}
\end{table}

% \usepackage{booktabs}
\begin{table}[t]
\caption{Results for Maximum coverage problem with uniform weights with same dispersion where the evaluation is based on Chebyshev's equality}
\label{table:uwd_cheb}
\resizebox{\textwidth}{0.5\textwidth}{
\begin{tabular}{@{}llllllllllllllll@{}}
\toprule
             &       &           &          & \multicolumn{3}{l}{\gs   (9)}           & \multicolumn{3}{l}{\sg (10)}               & \multicolumn{3}{l}{$\nsga_{20}$ (11)}    & \multicolumn{3}{l}{$\nsga_{100}$ (12)}       \\ \midrule
Graph        & $B$   & $t_{max}$ & $\alpha$ & Mean         & std    & stat              & Mean              & std   & stat             & Mean         & std    & stat             & Mean             & std    & stat             \\
ca-CSphd     & 43    & 1500000   & 0.1      & \textbf{38}  & 0      & 10(=),11(=),12(=) & \textbf{38}       & 0     & 9(=),11(=),12(=) & \textbf{38}  & 0      & 9(=),10(=),12(=) & \textbf{38}      & 0      & 9(=),10(=),11(=) \\
             &       &           & 0.001    & \textbf{22}  & 0      & 10(=),11(=),12(=) & \textbf{22}       & 0     & 9(=),11(=),12(=) & \textbf{22}  & 0      & 9(=),10(=),12(=) & \textbf{22}      & 0      & 9(=),10(=),11(=) \\
             &       & 1000000   & 0.1      & \textbf{38}  & 0      & 10(=),11(=),12(=) & \textbf{38}       & 0     & 9(=),11(=),12(=) & \textbf{38}  & 0      & 9(=),10(=),12(=) & \textbf{38}      & 0      & 9(=),10(=),11(=) \\
             &       &           & 0.001    & \textbf{22}  & 0      & 10(=),11(=),12(=) & \textbf{22}       & 0     & 9(=),11(=),12(=) & \textbf{22}  & 0      & 9(=),10(=),12(=) & \textbf{22}      & 0      & 9(=),10(=),11(=) \\
             &       & 500000    & 0.1      & \textbf{38}  & 0      & 10(=),11(=),12(=) & \textbf{38}       & 0     & 9(=),11(=),12(=) & \textbf{38}  & 0      & 9(=),10(=),12(=) & \textbf{38}      & 0      & 9(=),10(=),11(=) \\
             &       &           & 0.001    & \textbf{22}  & 0      & 10(=),11(=),12(=) & \textbf{22}       & 0     & 9(=),11(=),12(=) & \textbf{22}  & 0      & 9(=),10(=),12(=) & \textbf{22}      & 0      & 9(=),10(=),11(=) \\
             & 94    & 1500000   & 0.1      & \textbf{88}  & 0      & 10(=),11(=),12(=) & \textbf{88}       & 0     & 9(=),11(=),12(=) & 87.733       & 0.442  & 9(=),10(=),12(=) & \textbf{88}      & 0      & 9(=),10(=),11(=) \\
             &       &           & 0.001    & \textbf{65}  & 0      & 10(=),11(=),12(=) & \textbf{65}       & 0     & 9(=),11(=),12(=) & \textbf{65}  & 0      & 9(=),10(=),12(=) & \textbf{65}      & 0      & 9(=),10(=),11(=) \\
             &       & 1000000   & 0.1      & \textbf{88}  & 0      & 10(=),11(=),12(=) & \textbf{88}       & 0     & 9(=),11(=),12(=) & 87.7         & 0.458  & 9(=),10(=),12(=) & \textbf{88}      & 0      & 9(=),10(=),11(=) \\
             &       &           & 0.001    & \textbf{65}  & 0      & 10(=),11(=),12(=) & \textbf{65}       & 0     & 9(=),11(=),12(=) & \textbf{65}  & 0      & 9(=),10(=),12(=) & \textbf{65}      & 0      & 9(=),10(=),11(=) \\
             &       & 500000    & 0.1      & \textbf{88}  & 0      & 10(-),11(=),12(=) & \textbf{88}       & 0     & 9(+),11(+),12(+) & 87.633       & 5.467  & 9(=),10(-),12(=) & 87.966           & 0.179  & 9(=),10(-),11(=) \\
             &       &           & 0.001    & \textbf{65}  & 0      & 10(=),11(=),12(=) & \textbf{65}       & 0     & 9(=),11(=),12(=) & \textbf{65}  & 0      & 9(=),10(=),12(=) & \textbf{65}      & 0      & 9(=),10(=),11(=) \\
             & 188   & 1500000   & 0.1      & \textbf{175} & 0      & 10(=),11(+),12(=) & \textbf{175}      & 0     & 9(=),11(+),12(=) & 172.733      & 0.928  & 9(-),10(-),12(-) & \textbf{175}     & 0      & 9(=),10(=),11(+) \\
             &       &           & 0.001    & \textbf{137} & 0      & 10(=),11(=),12(=) & \textbf{137}      & 0     & 9(=),11(=),12(=) & \textbf{137} & 0      & 9(=),10(=),12(=) & \textbf{137}     & 0      & 9(=),10(=),11(=) \\
             &       & 1000000   & 0.1      & \textbf{175} & 0      & 10(=),11(+),12(=) & \textbf{175}      & 0     & 9(=),11(=),12(=) & 172.4        & 0.84   & 9(-),10(-),12(-) & \textbf{175}     & 0      & 9(=),10(=),11(+) \\
             &       &           & 0.001    & \textbf{137} & 0      & 10(=),11(=),12(=) & \textbf{137}      & 0     & 9(=),11(=),12(=) & \textbf{137} & 0      & 9(=),10(=),12(=) & \textbf{137}     & 0      & 9(=),10(=),11(=) \\
             &       & 500000    & 0.1      & 174.966      & 0.179  & 10(=),11(+),12(=) & \textbf{175}      & 0     & 9(=),11(=),12(=) & 171.466      & 0.884  & 9(-),10(-),12(-) & \textbf{175}     & 0      & 9(=),10(=),11(+) \\
             &       &           & 0.001    & \textbf{137} & 0      & 10(=),11(=),12(=) & \textbf{137}      & 0     & 9(=),11(=),12(=) & \textbf{137} & 0      & 9(=),10(=),12(=) & \textbf{137}     & 0      & 9(=),10(=),11(=) \\ \midrule
ca-GrQc      & 64    & 1500000   & 0.1      & 60.966       & 0.179  & 10(=),11(=),12(=) & \textbf{61}       & 0     & 9(=),11(=),12(=) & 60.933       & 0.249  & 9(=),10(=),12(=) & \textbf{61}      & 0      & 9(=),10(=),11(=) \\
             &       &           & 0.001    & \textbf{44}  & 0      & 10(=),11(=),12(=) & \textbf{44}       & 0     & 9(=),11(=),12(=) & \textbf{44}  & 0      & 9(=),10(=),12(=) & \textbf{44}      & 0      & 9(=),10(=),11(=) \\
             &       & 1000000   & 0.1      & 60.9         & 0.3    & 10(=),11(=),12(=) & \textbf{61}       & 0     & 9(=),11(=),12(=) & 60.766       & 0.422  & 9(=),10(=),12(=) & \textbf{61}      & 0      & 9(=),10(=),11(=) \\
             &       &           & 0.001    & \textbf{44}  & 0      & 10(=),11(=),12(=) & \textbf{44}       & 0     & 9(=),11(=),12(=) & \textbf{44}  & 0      & 9(=),10(=),12(=) & \textbf{44}      & 0      & 9(=),10(=),11(=) \\
             &       & 500000    & 0.1      & 60.566       & 0.667  & 10(=),11(=),12(=) & 60.933            & 0.249 & 9(=),11(=),12(=) & 60.433       & 0.76   & 9(=),10(=),12(=) & \textbf{61}      & 0      & 9(=),10(=),11(=) \\
             &       &           & 0.001    & 43.9         & 0.3    & 10(=),11(=),12(=) & \textbf{44}       & 0     & 9(=),11(=),12(=) & \textbf{44}  & 0      & 9(=),10(=),12(=) & \textbf{44}      & 0      & 9(=),10(=),11(=) \\
             & 207   & 1500000   & 0.1      & \textbf{199} & 0      & 10(=),11(=),12(=) & \textbf{199}      & 0     & 9(=),11(=),12(=) & 198          & 0.7745 & 9(=),10(=),12(=) & \textbf{199}     & 0      & 9(=),10(=),11(=) \\
             &       &           & 0.001    & 171.833      & 0.372  & 10(=),11(=),12(=) & \textbf{172}      & 0     & 9(=),11(=),12(=) & 171.766      & 0.495  & 9(=),10(=),12(=) & \textbf{172}     & 0      & 9(=),10(=),11(=) \\
             &       & 1000000   & 0.1      & \textbf{199} & 0      & 10(=),11(-),12(=) & \textbf{199}      & 0     & 9(=),11(+),12(=) & 197.766      & 0.882  & 9(-),10(-),12(-) & \textbf{199}     & 0      & 9(=),10(=),11(+) \\
             &       &           & 0.001    & 171.5        & 0.806  & 10(=),11(=),12(=) & \textbf{172}      & 0     & 9(=),11(=),12(=) & 171.466      & 1.175  & 9(=),10(=),12(=) & \textbf{172}     & 0      & 9(=),10(=),11(=) \\
             &       & 500000    & 0.1      & 198.766      & 0.422  & 10(=),11(=),12(=) & \textbf{199}      & 0     & 9(=),11(=),12(=) & 197.266      & 1.236  & 9(=),10(=),12(=) & 198.96           & 0.179  & 9(=),10(=),11(=) \\
             &       &           & 0.001    & 169.466      & 1.707  & 10(=),11(-),12(-) & 171.966           & 0.179 & 9(=),11(-),12(-) & 171.066      & 1.364  & 9(=),10(=),12(=) & \textbf{172}     & 0      & 9(+),10(+),11(=) \\
             & 415   & 1500000   & 0.1      & 398.266      & 0.442  & 10(=),11(+),12(=) & \textbf{399}      & 0     & 9(=),11(+),12(=) & 390.966      & 1.622  & 9(-),10(-),12(-) & 398.466          & 0.498  & 9(=),10(=),11(+) \\
             &       &           & 0.001    & 353.066      & 1.436  & 10(=),11(+),12(=) & \textbf{355}      & 0     & 9(=),11(+),12(=) & 346.366      & 2.676  & 9(-),10(-),12(-) & 354.533          & 0.498  & 9(=),10(=),11(+) \\
             &       & 1000000   & 0.1      & 397.933      & 0.442  & 10(=),11(+),12(=) & \textbf{399}      & 0     & 9(=),11(+),12(=) & 390.066      & 1.931  & 9(-),10(-),12(-) & 398.366          & 0.546  & 9(=),10(=),11(+) \\
             &       &           & 0.001    & 351.7        & 1.159  & 10(-),11(+),12(-) & \textbf{354.9}    & 0.3   & 9(+),11(+),12(=) & 345.733      & 2.379  & 9(-),10(-),12(-) & 354.5            & 0.562  & 9(+),10(=),11(+) \\
             &       & 500000    & 0.1      & 396.933      & 0.512  & 10(-),11(+),12(=) & \textbf{398.766}  & 0.422 & 9(+),11(+),12(=) & 388.333      & 2.102  & 9(-),10(-),12(-) & 397.866          & 0.6699 & 9(+),10(=),11(+) \\
             &       &           & 0.001    & 350          & 1.181  & 10(-),11(+),12(-) & \textbf{354.533}  & 0.669 & 9(+),11(+),12(-) & 344.066      & 2.644  & 9(-),10(-),12(-) & 354.333          & 0.596  & 9(+),10(+),11(+) \\
ca-CondaMat  & 146   & 1500000   & 0.1      & 141.266      & 0.679  & 10(=),11(=),12(=) & \textbf{141.8}    & 0.979 & 9(=),11(=),12(=) & 141.266      & 0.813  & 9(=),10(=),12(=) & \textbf{141.8}   & 0.979  & 9(=),10(=),11(=) \\ \midrule
\multicolumn{2}{l}{} &           & 0.001    & 122.9        & 3.014  & 10(=),11(-),12(-) & 125.933           & 0.249 & 9(=),11(-),12(-) & 125.7        & 0.458  & 9(+),10(+),12(=) & \textbf{126}     & 0      & 9(+),10(+),11(=) \\
             &       & 1000000   & 0.1      & 141.2        & 0.6    & 10(=),11(=),12(=) & \textbf{141.733}  & 0.963 & 9(=),11(=),12(=) & 141          & 0.632  & 9(=),10(=),12(=) & 141.533          & 0.884  & 9(=),10(=),11(=) \\
             &       &           & 0.001    & 122.23       & 3.051  & 10(=),11(-),12(-) & 125.9             & 0.3   & 9(=),11(-),12(-) & 125.366      & 1.048  & 9(+),10(+),12(=) & \textbf{126}     & 0      & 9(+),10(+),11(=) \\
             &       & 500000    & 0.1      & 141.133      & 0.498  & 10(=),11(=),12(=) & \textbf{141.8}    & 0.979 & 9(=),11(=),12(=) & 140.833      & 0.734  & 9(=),10(=),12(=) & 141.333          & 0.745  & 9(=),10(=),11(=) \\
             &       &           & 0.001    & 120.733      & 3.172  & 10(=),11(-),12(-) & 125.6             & 0.663 & 9(=),11(-),12(-) & 124.066      & 2.657  & 9(+),10(+),12(=) & \textbf{125.7}   & 0.458  & 9(+),10(+),11(=) \\
             & 1068  & 1500000   & 0.1      & 1037.266     & 1.093  & 10(-),11(+),12(=) & \textbf{1044.833} & 0.933 & 9(+),11(+),12(+) & 1015.333     & 4.706  & 9(-),10(-),12(-) & 1037.833         & 2.646  & 9(=),10(-),12(+) \\
             &       &           & 0.001    & 978.4        & 2.751  & 10(-),11(+),12(-) & 991.766           & 2.347 & 9(+),11(+),12(-) & 959.566      & 10.892 & 9(-),10(-),12(-) & \textbf{992.2}   & 3.664  & 9(+),10(+),11(+) \\
             &       & 1000000   & 0.1      & 1034.933     & 1.152  & 10(-),11(+),12(=) & \textbf{1044.133} & 1.231 & 9(+),11(+),12(+) & 1012.333     & 5.204  & 9(-),10(-),12(-) & 1036.833         & 2.956  & 9(=),10(-),12(+) \\
             &       &           & 0.001    & 975.1        & 2.3288 & 10(-),11(+),12(-) & 989.5             & 2.202 & 9(+),11(+),12(-) & 954.866      & 10.616 & 9(-),10(-),12(-) & \textbf{991.3}   & 3.377  & 9(+),10(+),11(+) \\
             &       & 500000    & 0.1      & 1030.833     & 1.293  & 10(-),11(+),12(+) & \textbf{1041.633} & 1.251 & 9(+),11(+),12(+) & 1008.633     & 6.441  & 9(-),10(-),12(-) & 1035.233         & 2.641  & 9(+),10(-),12(+) \\
             &       &           & 0.001    & 967.666      & 3.418  & 10(-),11(+),12(-) & 985.9             & 2.748 & 9(+),11(+),12(-) & 947.233      & 10.932 & 9(-),10(-),12(-) & \textbf{988.866} & 4.145  & 9(+),10(+),11(+) \\
             & 2136  & 1500000   & 0.1      & 2035.066     & 2.92   & 10(-),11(+),12(+) & \textbf{2071.066} & 1.412 & 9(+),11(+),12(+) & 1963.4       & 9.844  & 9(-),10(-),12(-) & 2025.6           & 5.689  & 9(+),10(-),11(+) \\
             &       &           & 0.001    & 1925.3       & 3.671  & 10(-),11(+),12(-) & \textbf{1972.433} & 3.402 & 9(+),11(+),12(+) & 1850.633     & 13.345 & 9(-),10(-),12(-) & 1942.566         & 6.189  & 9(+),10(-),11(+) \\
             &       & 1000000   & 0.1      & 2026.966     & 3.341  & 10(-),11(+),12(+) & \textbf{2068.033} & 1.905 & 9(+),11(+),12(+) & 1956.3       & 9.987  & 9(-),10(-),12(-) & 2022.6           & 6.58   & 9(+),10(-),11(+) \\
             &       &           & 0.001    & 1914.7       & 3.831  & 10(-),11(+),12(-) & \textbf{1969.433} & 3.666 & 9(+),11(+),12(+) & 1839.766     & 13.313 & 9(-),10(-),12(-) & 1939.833         & 7.55   & 9(+),10(-),11(+) \\
             &       & 500000    & 0.1      & 2009.366     & 4.214  & 10(-),11(+),12(-) & \textbf{2063.166} & 2.852 & 9(+),11(+),12(+) & 1941.3       & 11.346 & 9(-),10(-),12(-) & 2016.5           & 5.942  & 9(+),10(-),11(+) \\
             &       &           & 0.001    & 1893.5       & 4.055  & 10(-),11(+),12(-) & \textbf{1960}     & 3.705 & 9(+),11(+),12(+) & 1821         & 14.61  & 9(-),10(-),12(-) & 1931.866         & 6.443  & 9(+),10(-),11(+) \\ \bottomrule
\end{tabular}
}
\end{table}

% \usepackage{booktabs}
\begin{table}[]
\caption{Results for Maximum coverage problem with uniform weights with same dispersion where the evaluation is based on Chernoff bound}
\label{table:uwd_chf}
\resizebox{\textwidth}{0.5\textwidth}{
\begin{tabular}{@{}llllllllllllllll@{}}
\toprule
            &      &           &          & \multicolumn{3}{l}{\gs (13)}          & \multicolumn{3}{l}{\sg (14)}                & \multicolumn{3}{l}{$\nsga_{20}$ (15)}     & \multicolumn{3}{l}{$\nsga_{100}$ (16)}     \\ \midrule
Graph       & $B$  & $t_{max}$ & $\alpha$ & Mean         & std    & stat              & Mean              & std   & stat              & Mean         & std    & stat              & Mean           & std   & stat              \\
ca-CSphd    & 43   & 1500000   & 0.1      & \textbf{36}  & 0      & 14(=),15(=),16(=) & \textbf{36}       & 0     & 13(=),15(=),16(=) & \textbf{36}  & 0      & 13(=),14(=),16(=) & \textbf{36}    & 0     & 13(=),14(=),15(=) \\
            &      &           & 0.001    & \textbf{33}  & 0      & 14(=),15(=),16(=) & \textbf{33}       & 0     & 13(=),15(=),16(=) & \textbf{33}  & 0      & 13(=),14(=),16(=) & \textbf{33}    & 0     & 13(=),14(=),15(=) \\
            &      & 1000000   & 0.1      & \textbf{36}  & 0      & 14(=),15(=),16(=) & \textbf{36}       & 0     & 13(=),15(=),16(=) & \textbf{36}  & 0      & 13(=),14(=),16(=) & \textbf{36}    & 0     & 13(=),14(=),15(=) \\
            &      &           & 0.001    & \textbf{33}  & 0      & 14(=),15(=),16(=) & \textbf{33}       & 0     & 13(=),15(=),16(=) & \textbf{33}  & 0      & 13(=),14(=),16(=) & \textbf{33}    & 0     & 13(=),14(=),15(=) \\
            &      & 500000    & 0.1      & \textbf{36}  & 0      & 14(=),15(=),16(=) & \textbf{36}       & 0     & 13(=),15(=),16(=) & \textbf{36}  & 0      & 13(=),14(=),16(=) & \textbf{36}    & 0     & 13(=),14(=),15(=) \\
            &      &           & 0.001    & \textbf{33}  & 0      & 14(=),15(=),16(=) & \textbf{33}       & 0     & 13(=),15(=),16(=) & \textbf{33}  & 0      & 13(=),14(=),16(=) & \textbf{33}    & 0     & 13(=),14(=),15(=) \\
            & 94   & 1500000   & 0.1      & \textbf{85}  & 0      & 14(=),15(=),16(=) & \textbf{85}       & 0     & 13(=),15(=),16(=) & \textbf{85}  & 0      & 13(=),14(=),16(=) & \textbf{85}    & 0     & 13(=),14(=),15(=) \\
            &      &           & 0.001    & \textbf{81}  & 0      & 14(=),15(=),16(=) & \textbf{81}       & 0     & 13(=),15(=),16(=) & \textbf{81}  & 0      & 13(=),14(=),16(=) & \textbf{81}    & 0     & 13(=),14(=),15(=) \\
            &      & 1000000   & 0.1      & \textbf{85}  & 0      & 14(=),15(=),16(=) & \textbf{85}       & 0     & 13(=),15(=),16(=) & \textbf{85}  & 0      & 13(=),14(=),16(=) & \textbf{85}    & 0     & 13(=),14(=),15(=) \\
            &      &           & 0.001    & \textbf{81}  & 0      & 14(=),15(=),16(=) & \textbf{81}       & 0     & 13(=),15(=),16(=) & \textbf{81}  & 0      & 13(=),14(=),16(=) & \textbf{81}    & 0     & 13(=),14(=),15(=) \\
            &      & 500000    & 0.1      & \textbf{85}  & 0      & 14(=),15(=),16(=) & \textbf{85}       & 0     & 13(=),15(=),16(=) & \textbf{85}  & 0      & 13(=),14(=),16(=) & \textbf{85}    & 0     & 13(=),14(=),15(=) \\
            &      &           & 0.001    & \textbf{81}  & 0      & 14(=),15(=),16(=) & \textbf{81}       & 0     & 13(=),15(=),16(=) & \textbf{81}  & 0      & 13(=),14(=),16(=) & \textbf{81}    & 0     & 13(=),14(=),15(=) \\
            & 188  & 1500000   & 0.1      & \textbf{171} & 0      & 14(=),15(+),16(=) & \textbf{171}      & 0     & 13(=),15(+),16(=) & 169.866      & 1.11   & 13(-),14(-),16(-) & \textbf{171}   & 0     & 13(=),14(=),15(+) \\
            &      &           & 0.001    & \textbf{166} & 0      & 14(=),15(+),16(=) & \textbf{166}      & 0     & 13(=),15(+),16(=) & 164.3        & 0.69   & 13(-),14(-),16(-) & \textbf{166}   & 0     & 13(=),14(=),15(+) \\
            &      & 1000000   & 0.1      & \textbf{171} & 0      & 14(=),15(=),16(=) & \textbf{171}      & 0     & 13(=),15(+),16(=) & 169.6        & 1.019  & 13(-),14(-),16(-) & \textbf{171}   & 0     & 13(=),14(=),15(+) \\
            &      &           & 0.001    & \textbf{166} & 0      & 14(=),15(+),16(=) & \textbf{166}      & 0     & 13(=),15(+),16(=) & 163.933      & 1.062  & 13(-),14(-),16(-) & \textbf{166}   & 0     & 13(=),14(=),15(+) \\
            &      & 500000    & 0.1      & \textbf{171} & 0      & 14(=),15(+),16(=) & \textbf{171}      & 0     & 13(=),15(+),16(=) & 168.566      & 1.054  & 13(-),14(-),16(-) & \textbf{171}   & 0     & 13(=),14(=),15(+) \\
            &      &           & 0.001    & 165.366      & 0.795  & 14(=),15(+),16(=) & \textbf{166}      & 0     & 13(=),15(+),16(=) & 163.2        & 1.301  & 13(-),14(-),16(-) & \textbf{166}   & 0     & 13(=),14(=),15(+) \\ \midrule
ca-GrQc     & 64   & 1500000   & 0.1      & 57.966       & 0.1795 & 14(=),15(=),16(=) & \textbf{58}       & 0     & 13(=),15(=),16(=) & \textbf{58}  & 0      & 13(=),14(=),16(=) & \textbf{58}    & 0     & 13(=),14(=),15(=) \\
            &      &           & 0.001    & \textbf{57}  & 0      & 14(=),15(=),16(=) & \textbf{57}       & 0     & 13(=),15(=),16(=) & \textbf{57}  & 0      & 13(=),14(=),16(=) & \textbf{57}    & 0     & 13(=),14(=),15(=) \\
            &      & 1000000   & 0.1      & 57.966       & 0.179  & 14(=),15(=),16(=) & \textbf{58}       & 0     & 13(=),15(=),16(=) & 57.966       & 0.179  & 13(=),14(=),16(=) & \textbf{58}    & 0     & 13(=),14(=),15(=) \\
            &      &           & 0.001    & \textbf{57}  & 0      & 14(=),15(=),16(=) & \textbf{57}       & 0     & 13(=),15(=),16(=) & \textbf{57}  & 0      & 13(=),14(=),16(=) & \textbf{57}    & 0     & 13(=),14(=),15(=) \\
            &      & 500000    & 0.1      & 57.733       & 0.442  & 14(=),15(=),16(=) & \textbf{58}       & 0     & 13(=),15(=),16(=) & 57.766       & 0.422  & 13(=),14(=),16(=) & \textbf{58}    & 0     & 13(=),14(=),15(=) \\
            &      &           & 0.001    & 56.766       & 0.667  & 14(=),15(=),16(=) & \textbf{57}       & 0     & 13(=),15(=),16(=) & \textbf{57}  & 0      & 13(=),14(=),16(=) & \textbf{57}    & 0     & 13(=),14(=),15(=) \\
            & 207  & 1500000   & 0.1      & \textbf{196} & 0      & 14(=),15(=),16(=) & \textbf{196}      & 0     & 13(=),15(=),16(=) & 194.466      & 1.024  & 13(=),14(=),16(=) & \textbf{196}   & 0     & 13(=),14(=),15(=) \\
            &      &           & 0.001    & \textbf{192} & 0      & 14(=),15(=),16(=) & \textbf{192}      & 0     & 13(=),15(=),16(=) & 191.5        & 1.024  & 13(=),14(=),16(=) & \textbf{192}   & 0     & 13(=),14(=),15(=) \\
            &      & 1000000   & 0.1      & \textbf{196} & 0      & 14(=),15(=),16(=) & \textbf{196}      & 0     & 13(=),15(=),16(=) & 194.266      & 0.963  & 13(=),14(=),16(=) & \textbf{196}   & 0     & 13(=),14(=),15(=) \\
            &      &           & 0.001    & \textbf{192} & 0      & 14(=),15(=),16(=) & \textbf{192}      & 0     & 13(=),15(=),16(=) & 191.3        & 1.037  & 13(=),14(=),16(=) & \textbf{192}   & 0     & 13(=),14(=),15(=) \\
            &      & 500000    & 0.1      & 195.833      & 0.372  & 14(=),15(+),16(=) & \textbf{196}      & 0     & 13(=),15(+),16(=) & 193.966      & 0.835  & 13(-),14(-),16(-) & \textbf{196}   & 0     & 13(=),14(=),15(+) \\
            &      &           & 0.001    & 191.966      & 0.179  & 14(=),15(+),16(=) & \textbf{192}      & 0     & 13(=),15(+),16(=) & 190.266      & 1.412  & 13(-),14(-),16(-) & \textbf{192}   & 0     & 13(=),14(=),15(+) \\
            & 415  & 1500000   & 0.1      & 394.066      & 0.771  & 14(=),15(+),16(=) & \textbf{395.333}  & 0.471 & 13(=),15(+),16(=) & 386.833      & 1.694  & 13(-),14(-),16(-) & 394.8          & 0.476 & 13(=),14(=),15(+) \\
            &      &           & 0.001    & 386.866      & 0.339  & 14(=),15(+),16(=) & \textbf{388.833}  & 0.372 & 13(=),15(+),16(=) & 379.966      & 2.575  & 13(-),14(-),16(-) & 387.533        & 0.845 & 13(=),14(=),15(+) \\
            &      & 1000000   & 0.1      & 393.333      & 0.442  & 14(=),15(+),16(=) & \textbf{395.166}  & 0.372 & 13(=),15(+),16(=) & 386.133      & 2.124  & 13(-),14(-),16(-) & 394.333        & 0.829 & 13(=),14(=),15(+) \\
            &      &           & 0.001    & 386.233      & 0.76   & 14(=),15(+),16(=) & \textbf{388.333}  & 0.869 & 13(=),15(+),16(=) & 378.8        & 2.508  & 13(-),14(-),16(-) & 387.366        & 0.752 & 13(=),14(=),15(+) \\
            &      & 500000    & 0.1      & 392.233      & 0.512  & 14(=),15(+),16(=) & \textbf{394.866}  & 0.426 & 13(=),15(+),16(=) & 383.633      & 2.272  & 13(-),14(-),16(-) & 393.8          & 0.945 & 13(=),14(=),15(+) \\
            &      &           & 0.001    & 384.7        & 0.525  & 14(-),15(+),16(-) & \textbf{387.566}  & 0.803 & 13(=),15(+),16(=) & 376.7        & 2.223  & 13(-),14(-),16(-) & 386.766        & 0.989 & 13(-),14(=),15(+) \\ \midrule
ca-CondaMat & 146  & 1500000   & 0.1      & \textbf{139} & 0      & 14(=),15(=),16(=) & \textbf{139}      & 0     & 13(=),15(=),16(=) & \textbf{139} & 0      & 13(=),14(=),16(=) & \textbf{139}   & 0     & 13(=),14(=),15(=) \\
            &      &           & 0.001    & 137.1        & 1.445  & 14(=),15(=),16(=) & 138.2             & 1.326 & 13(=),15(=),16(=) & 136.9        & 1.374  & 13(=),14(=),16(=) & \textbf{138.3} & 1.268 & 13(=),14(=),15(=) \\
            &      & 1000000   & 0.1      & \textbf{139} & 0      & 14(=),15(=),16(=) & \textbf{139}      & 0     & 13(=),15(=),16(=) & \textbf{139} & 0      & 13(=),14(=),16(=) & \textbf{139}   & 0     & 13(=),14(=),15(=) \\
            &      &           & 0.001    & 136.7        & 1.1268 & 14(=),15(=),16(=) & 137.8             & 1.469 & 13(=),15(=),16(=) & 136.5        & 1.118  & 13(=),14(=),16(=) & \textbf{138}   & 1.414 & 13(=),14(=),15(=) \\
            &      & 500000    & 0.1      & \textbf{139} & 0      & 14(=),15(=),16(=) & \textbf{139}      & 0     & 13(=),15(=),16(=) & 138.966      & 0.179  & 13(=),14(=),16(=) & \textbf{139}   & 0     & 13(=),14(=),15(=) \\
            &      &           & 0.001    & 136.4        & 1.019  & 14(=),15(=),16(=) & 136.8             & 1.326 & 13(=),15(=),16(=) & 136.266      & 0.928  & 13(=),14(=),16(=) & \textbf{137.1} & 1.445 & 13(=),14(=),15(=) \\
            & 1068 & 1500000   & 0.1      & 1031.266     & 1.59   & 14(-),15(+),16(=) & \textbf{1039.366} & 1.139 & 13(+),15(+),16(+) & 1008.4       & 5.505  & 13(-),14(-),16(-) & 1033.966       & 2.994 & 13(=),14(-),15(+) \\
            &      &           & 0.001    & 1022.533     & 1.726  & 14(-),15(+),16(-) & \textbf{1031.533} & 0.956 & 13(+),15(+),16(+) & 1001.133     & 5.754  & 13(-),14(-),16(-) & 1026.966       & 2.575 & 13(+),14(-),15(+) \\
            &      & 1000000   & 0.1      & 1029.066     & 1.31   & 14(-),15(+),16(-) & \textbf{1038.166} & 1.097 & 13(+),15(+),16(+) & 1005.833     & 6.044  & 13(-),14(-),16(-) & 1033.266       & 3.203 & 13(+),14(-),15(+) \\
            &      &           & 0.001    & 1019.466     & 1.783  & 14(-),15(+),16(-) & \textbf{1030.366} & 1.425 & 13(+),15(+),16(=) & 997.466      & 6.463  & 13(-),14(-),16(-) & 1026.066       & 2.249 & 13(+),14(=),15(+) \\
            &      & 500000    & 0.1      & 1024.633     & 1.87   & 14(-),15(+),16(-) & \textbf{1036.133} & 0.956 & 13(+),15(+),16(+) & 1000.7       & 7.299  & 13(-),14(-),16(-) & 1031.766       & 3.921 & 13(+),14(-),15(+) \\
            &      &           & 0.001    & 1014.1       & 2.211  & 14(-),15(+),16(-) & \textbf{1027.833} & 1.507 & 13(+),15(+),16(=) & 992.5        & 7.069  & 13(-),14(-),16(-) & 1024.866       & 2.459 & 13(+),14(=),15(+) \\
            & 2136 & 1500000   & 0.1      & 2023.533     & 2.704  & 14(-),15(+),16(+) & \textbf{2062.166} & 1.881 & 13(+),15(+),16(+) & 1949.366     & 9.064  & 13(-),14(-),16(-) & 2019.5         & 6.687 & 13(+),14(-),15(+) \\
            &      &           & 0.001    & 2006.533     & 2.376  & 14(-),15(+),16(+) & \textbf{2041.2}   & 2.072 & 13(+),15(+),16(+) & 1932.4       & 10.694 & 13(-),14(-),16(-) & 2003.3         & 5.502 & 13(+),14(-),15(+) \\
            &      & 1000000   & 0.1      & 2015.366     & 3.219  & 14(-),15(+),16(=) & \textbf{2059.433} & 1.994 & 13(+),15(+),16(+) & 1942.266     & 9.051  & 13(-),14(-),16(-) & 2014.9         & 5.497 & 13(+),14(-),15(+) \\
            &      &           & 0.001    & 1997.5       & 3.232  & 14(-),15(+),16(-) & \textbf{2046.233} & 1.977 & 13(+),15(+),16(+) & 1924.566     & 11.632 & 13(-),14(-),16(-) & 2000.533       & 6.173 & 13(+),14(-),15(+) \\
            &      & 500000    & 0.1      & 1995.8       & 3.187  & 14(-),15(+),16(-) & \textbf{2052.833} & 2.296 & 13(+),15(+),16(+) & 1929.9       & 10.077 & 13(-),14(-),16(-) & 2008.766       & 6.897 & 13(+),14(-),15(+) \\
            &      &           & 0.001    & 1977.633     & 2.857  & 14(-),15(+),16(-) & \textbf{2037.933} & 2.555 & 13(+),15(+),16(+) & 1909.133     & 12.164 & 13(-),14(-),16(-) & 1993.9         & 7.449 & 13(+),14(-),15(+) \\ \bottomrule
\end{tabular}
}
\end{table}

\end{document}